\documentclass[10pt,a4paper]{article}                     % onecolumn (standard format)
\usepackage{graphicx}
\usepackage[fleqn]{amsmath}
\usepackage{amssymb}
\usepackage{amsthm}
\usepackage{enumerate}
\usepackage{xspace}
\usepackage{url}
\usepackage{natbib}
\usepackage{peterdefs}
\usepackage{rotating}

\ifnum\pdfoutput>0
\else
\usepackage[active]{srcltx} %% for xdvi
\fi

\newtheorem{theorem}{Theorem}
\newtheorem{proposition}{Proposition}
\newtheorem{corollary}{Corollary}

\newcommand{\Omit}[1]{}

\newcommand{\calH}{{\cal H}}
\newcommand{\calB}{{\cal B}}
\newcommand{\E}{{\ensuremath{\mathrm{E}}}}

\newcommand{\NOT}[1]{{\mathsf{not\_}#1}}
\newcommand{\blsd}{\text{sdb}} %% subterm domain blocking 
\newcommand{\blsp}{\text{spb}} %% subterm predicate blocking
\newcommand{\blud}{\text{udb}} %% unrestricted domain blocking
\newcommand{\blup}{\text{upb}} %% unrestricted predicate blocking
\newcommand{\rr}{\text{rr}} %% range-restriction
\newcommand{\crr}{\text{crr}} %% classical range-restriction
\newcommand{\sh}{\text{sh}} %% shifting
\newcommand{\pf}{\text{pf}} %% partial flattening
\newcommand{\bs}{\text{bs}} %% basic shifting
\newcommand{\tr}{\text{tr}} %% basic shifting
\newcommand{\EAX}{\text{EAX}}

\newcommand{\dom}{\ID{dom}}
\newcommand{\subterm}{\ID{sub}}

\newcommand{\myequal}{\ID{myequal}}

\newcommand{\BS}{\text{BS}\xspace}

\newcommand{\ba}{\begin{array}}
\newcommand{\ea}{\end{array}}

\begin{document}

\title{Blocking and Other Enhancements for
Bottom-Up Model Generation Methods%\thanks{Grants or other notes
}

\author{Peter Baumgartner \and
        Renate A.\ Schmidt}

\thanks{P. Baumgartner,
              Data61, CSIRO,
               Australia %\\
           \\
              R. A. Schmidt,
              School of Computer Science,
              The University of Manchester,
              UK\\
}

\date{23 November 2016}

\maketitle

\begin{abstract}
Model generation is a problem complementary to theorem proving
and is important for fault analysis and debugging of formal
specifications of, for example, security protocols, programs and terminological
definitions.
This paper discusses several ways of enhancing the paradigm of bottom-up model
generation.
The two main contributions are new, generalized blocking techniques 
and a new range-restriction transformation.
The blocking techniques are based on simple transformations of the
input set together with standard equality reasoning and redundancy
elimination techniques.
These provide general methods for finding small, finite models.
The range-restriction transformation refines existing transformations
to range-restricted clauses by carefully limiting the creation of domain terms.
All possible combinations of the introduced techniques and a classical
range-restriction technique were tested on the clausal problems
of the TPTP Version 6.0.0 with an implementation based on the SPASS
theorem prover using a hyperresolution-like refinement.
Unrestricted domain blocking gave best results for satisfiable problems
showing it is a powerful technique indispensable for bottom-up model
generation methods.
Both in combination with the new range-restricting transformation,
and the classical range-restricting transformation, good results have
been obtained.
Limiting the creation of terms during the inference process by using
the new range restricting transformation has paid off, especially when
using it together with a shifting transformation.
The experimental results also show that classical range
restriction with unrestricted blocking provides a useful complementary
method.
Overall, the results showed bottom-up model generation methods were good for disproving theorems
and generating models for satisfiable problems, but less efficient than
SPASS in auto mode for unsatisfiable problems.

%\keywords{automated reasoning \and model generation \and blocking \and first-order
%logic \and Bernays-Sch\"onfinkel class}
\end{abstract}

\section{Introduction}
\label{sec:introduction}

The bottom-up model generation (BUMG) paradigm encompasses a wide
family of calculi and proof procedures that explicitly try to construct
a model of a given clause set by reading clauses as
rules and applying them in a bottom-up way until completion.
For instance, variants of hyperresolution and grounding tableau calculi
belong to this family.
BUMG methods have been known for a
long time to useful for proving theorems, 
comparably little effort has however been undertaken to
exploit them for the dual task, namely,
computing models for satisfiable problems. 
This is somewhat surprising, as computing models is recognized as
being important in software engineering, model checking, and other
applications for fault analysis and debugging of logical specifications.

One of the contributions of the paper is the introduction to
first-order logic of blocking techniques partially
inspired by techniques already successfully used in
description and modal logic tableau-based theorem
proving~\citep{HustadtSchmidt99b,SchmidtTishkovsky07b,BaaderSattler01}.
We adapt and generalize these blocking techniques to full first-order
logic.
Blocking is an important technique for turning tableau systems into
decision procedures for modal and description logics.
Though different blocking techniques exist, and not all modal and
description logic tableau systems are designed to
return models, blocking is essentially a mechanism for
systematically merging terms in order to find finite models.

In our approach blocking is encoded on the clausal level and is
combined with standard resolution techniques, the idea being that
with a suitable prover small, finite models are constructed and can
be easily read off from the derived clauses.
Our blocking techniques are generic and pose no restrictions on the
logic they can be used for.
They can even be used for undecidable logics.
We introduce four different blocking techniques.
The main idea of our blocking techniques is that clauses are added to the input
problem which lead in the derivation to splittable clauses causing
terms in the partially constructed models to be merged.
The difference between the four techniques is how restrictive blocking is.
With \emph{unrestricted domain blocking} domain minimal models can be
generated. With \emph{subterm domain blocking} or \emph{subterm
predicate blocking} larger models are produced because two terms are only
merged if one is a subterm of the other.
With \emph{unrestricted predicate blocking} and \emph{subterm predicate
blocking} two terms are merged if they both belong to the extension of a
unary predicate symbol, the intention being that less constrained,
finite model can be found.

The second contribution of the paper is a refinement of
the well-known `transformation to range-restricted form' as
introduced in the eighties by~\citet{Manthey:Bry:SATCHMO:88}
in the context of the SATCHMO prover and later improved, for
example, by~\citet{Baumgartner:Furbach:Stolzenburg:RestartMEAnswers:AI:97}.
These range-restricting transformations have the disadvantage that they
generally force BUMG methods to enumerate the entire Herbrand universe
and are therefore non-terminating except in the simplest cases.
One solution is to combine classical range-restriction
transformations with blocking techniques.
Another solution, presented in this paper, is to modify the range-restricting transformation so
that new terms are created only when needed.
Our method extends and combines the range-restricting transformation
introduced in~\citet{SchmidtHustadt05b} for reducing first-order
formulae and clauses into range-restricted clauses, which was used
to develop general-purpose resolution decision procedures for the
Bernays-Sch\"onfinkel class.

Other methods for model computation can be classified as methods that
directly search for a finite model, such as the extended PUHR tableau
method of~\cite{BryTorge98},
the method of~\citet{Bezem:GeometricLogic:DisprovingWS:2005} and the
methods in the
SEM-family~\citep{Slaney:FINDER:TR:92,Zhang:SEM:IJCAI:95,McCune:MACE4:2003}.
In contrast, MACE-style model builders
such as, for example, the methods
of~\cite{Claessen:Soerensson:MACEimprove:ModelComputationWS:2003} and~\cite{McCune:DPFirstOrderModelSearch:TechRep:94}
reduce  model search to testing of propositional satisfiability. 
Being based on a translation, the MACE-style approach is 
conceptually related, but different to our approach. 
Both SEM- and MACE-style methods search for finite models, by 
essentially searching the 
space of interpretations with domain sizes $1,2,\ldots$, in increasing
order, until a model is found.

Our method operates significantly differently, as it is {\em not\/}
parameterized by a domain size.
Consequently, there is no requirement for iterative deepening over
the domain size, and the search for finite models works differently.
This way, we can address a problem often found
with models computed by these methods: from a pragmatic perspective,
they tend to identify too many terms.  For instance, for the two unit
clauses $\ID P(\ID a)$ and $\ID Q(\ID b)$ there is a model that
identifies $\ID a$ and $\ID b$ with the same object. Such models can
be counter-intuitive, for instance, in a description logic setting,
where unique names are often assumed, but not necessarily explicitly specified.
Furthermore, logic programs are typically understood with respect to Herbrand
semantics, and it is desirable to develop compatible model building
techniques.
We present transformations that are more careful at identifying
objects than the methods mentioned and thus work closer to a Herbrand
semantics.

The structure of the paper is as follows.
Definitions of basic terminology and notation can be found in
Section~\ref{sec:preliminaries}.
In Section~\ref{sec:BUMG} we recall the characteristic properties of
BUMG methods.
The main part of the paper are Sections~\ref{sec:transformations}
to~\ref{sec:experiments}.
Sections~\ref{sec:transformations}, \ref{sec:shifting}
and~\ref{sec:blocking} define new techniques for generating small
models and generating them more efficiently.
The techniques are based on a series of transformations including
a refined range-restricting transformation 
(Section~\ref{sec:range-restriction}), instances of standard
renaming and flattening (Section~\ref{sec:shifting}),
and the introduction of blocking in various forms through amendments of the clause set
and standard saturation-based equality reasoning
(Section~\ref{sec:blocking}).
Soundness and completeness of the blocking transformations and the
combined transformations is shown
in Section~\ref{sec:soundness:completeness}.
One consequence of the results is a general decidability result of the
Bernays-Sch\"onfinkel class for all
BUMG methods and related approaches.
This is presented in Section~\ref{sec:other}.
In Section~\ref{sec:experiments} we present and discuss results of
experiments carried out with our methods on clausal problems in
the TPTP library.

This paper is an extended and improved version
of~\citet{BaumgartnerSchmidt06}.

\section{Basic Definitions}
\label{sec:preliminaries}

We use standard terminology from automated reasoning. We assume as
given a signature $\Sigma = \Sigma_f \cup \Sigma_P$ of function
symbols~$\Sigma_f$ (including
constants) and predicate symbols~$\Sigma_P$. % of given arities. 
As we are working (also) with equality, we assume $\Sigma_P$ contains a
distinguished binary predicate symbol $\approx$, which is used in infix
form.  Terms, atoms, literals and formulas over $\Sigma$ and a given
(denumerable) set of variables $V$ are defined as usual.

A clause is a (finite) implicitly universally quantified disjunction
of literals. We write clauses in a logic-programming style,
that is, we write $H_1 \lor \cdots \lor H_m \gets B_1
\land \cdots \land B_k$ rather than $H_1 \lor \cdots \lor H_m \lor \lnot
B_1\lor \cdots \lor \lnot B_k$, where $m,k \geq 0$.
Each $H_i$ is called a {\em head atom\/}, and each $B_j$ is called a {\em body
  atom\/}.  When writing expressions such as $H \lor \calH \gets B \land \calB$ we
mean any clause whose head literals are $H$ and those in the
disjunction of literals $\calH$, and whose body literals are $B$ and
those in the conjunction of literals $\calB$.  A {\em clause set\/} is a
finite set of clauses.

A clause $\calH \gets \calB$ is said to be \emph{range-restricted}
iff the body~$\calB$ contains all the variables in it.
This means that a positive clause $\calH \gets \top$ is
range-restricted only if it is a ground clause.
A clause set is range-restricted iff it contains only range-restricted
clauses.

For a given atom $P(t_1,\ldots,t_n)$ the terms $t_1,\ldots,t_n$ are
also called the \emph{top-level terms} of $P(t_1,\ldots,t_n)$ ($P$ being
$\approx$ is permitted).
This notion generalizes to
clause bodies, clause heads and clauses as expected. For example, for a
clause $\calH \gets \calB$ the top-level terms of its body $\calB$ are
exactly the top-level terms of its body atoms.

A {\em proper functional term\/} is a term that is
  neither a variable nor a constant.

A \emph{(Herbrand) interpretation\/} $I$
is a set of ground atoms, namely, those
that are true in the interpretation.
Satisfiability/validity in a Herbrand interpretation of ground literals, clauses,
and clause sets
is defined as usual. Also, as usual, a
clause set stands semantically for the set of all its ground instances.
We write $I \models F$ to denote that $I$ satisfies $F$, where
$F$ is a ground literal or a (possibly non-ground) clause (set).

An \emph{E-interpretation} is an interpretation that is also a
congruence relation on the terms in the signature.
If $I$ is an interpretation, we denote by $I^E$ the smallest congruence
relation on the terms that includes $I$, which is an E-interpretation.
An E-interpretation does not necessarily need to be a
Herbrand-E-interpretation and is a standard first-order interpretation $I$
such that $(I,\mu) \models s \approx t$ if and only if $(I,\mu)(s) =
(I,\mu)(t)$ (where $\mu$ is a valuation, that is, a mapping from the
variables to the domain $|I|$ of $I$).
We say that $I$ {\em E-satisfies\/} $F$
iff $I^\E \models F$.
Instead of $I^\E \models F$ we write $I \models_\E F$.

It is well-known that E-interpretations can be characterized by
fixing the domain as the Herbrand universe and requiring that for every ground term~$t$, $t \approx t \in I$, and for every ground atom $A$ (including
ground equations) the following is true: whenever $I \models A[s]$
and $I \models s\approx t$, then $I \models
A[t]$.  

Another characterization is to add to a given clause set
$M$ its equality axioms $\EAX(\Sigma_P \cup \Sigma_f)$, that is, the axioms expressing that
$\approx$ is a congruence relation on the terms and atoms 
induced by the predicate symbols $\Sigma_P$ and function symbols $\Sigma_f$
occurring in $M$. It is well-known 
that $M$ is E-satisfiable iff $M \cup \EAX(\Sigma_P \cup \Sigma_f)$ is
satisfiable.

We work mostly, but not always, with Herbrand interpretations.  If
not, we always make this clear, and the interpretations
considered then are first-order logic interpretations with domains
that are (proper) subsets of the Herbrand universe of the clause
set under consideration. Such interpretations are called
{\em quasi-Herbrand interpretations\/}.
When constructing such interpretations the requirement that function
symbols are interpreted as total functions over their domain is not
always trivially satisfied. For instance, in the presence of a
constant $a$, a unary function symbol $f$, and the domain
$\{ a, f(a) \}$, say, one has to assign a value in the interpretation
to every term.
However $f(f(a))$, for instance, cannot be assigned to itself, as
$f(f(a))$ is not contained in the domain.

\section{BUMG Methods}
\label{sec:BUMG}

Proof procedures based on model generation approaches establish the
satisfiability of a problem by trying to build a model for the
problem.
In this paper we are interested in bottom-up model generation approaches (BUMG).
BUMG approaches use a forward reasoning approach where implications
or clauses, $\calH \gets \calB$, are read as rules and are repeatedly
used to derive (instances of)~$\calH$ from (instances of)~$\calB$
until a completion is found.

The family of BUMG approaches includes many familiar calculi and proof
procedures such as
Smullyan type semantic tableaux~\citep{Smullyan71},
SATCHMO \citep{Manthey:Bry:SATCHMO:88,Geisler:etal:CompilingFunctionalSatchmo:JAR:97},
positive unit hyperresolution (PUHR) tableaux~\citep{BryYahya00,BryTorge98},
the model generation theorem prover MGTP~\citep{Fujita:Slaney:Bennett:FiniteAlgebraMGTP:IJCAI:95} and  
hypertableaux~\citep{Baumgartner:Furbach:Niemelae:HyperTableau:JELIA:96}.
A well-established and widely known method for BUMG is
hyperresolution~\citep{Robinson65b}.

Hyperresolution consists of two inference rules, hyperresolution
and factoring.
The \emph{hyperresolution rule} applies to a non-positive clause
$H \gets B_1 \land \ldots \land B_n$ ($n \not=0$) and~$n$~positive clauses
$C_1 \lor B'_1 \gets \top$, \ldots, $C_n \lor B'_n \gets \top$, and
derives 
$(C_1 \lor \ldots \lor C_n \lor H)\sigma \gets \top$,
where~$\sigma$ is the most general unifier such that $B'_{i}\sigma =
B_{i}\sigma$ for every $i\in \{1,\ldots,n\}$.
The \emph{factoring rule} derives the clause $(C \lor B)\sigma \gets \top$
from a positive clause $C \lor B \lor B' \gets \top$, where $\sigma$ is the
most general unifier of $B$ and $B'$.
On range-restricted clauses, when using hyperresolution, factoring
amounts to the elimination of duplicate literals in positive clauses
and is therefore optional when clauses are viewed as sets.

A crucial requirement for the effective use of blocking
(considered later in Section~\ref{sec:blocking})
is support of equality reasoning (for
example, ordered paramodulation, ordered rewriting or
superposition~\citep{BachmairGanzinger98c,Nieuwenhuis:Rubio:ParamodulationTheoremProving:HandbookAR:2001}),
in combination with simplification techniques based on orderings.
We refer to~\citet{BachmairGanzinger98c,BachmairGanzinger01a} for 
general notions of redundancy 
in saturation-based theorem proving approaches.

Our experiments show the splitting rule is useful for BUMG.
For our blocking transformations, splitting on the positive part of
(ground) clauses is in fact mandatory to make it effective. 
This type of splitting replaces the branch of a derivation
containing a positive clause $C \lor D \gets \top$, say, by two
copies of the branch in which the clause is replaced by $C \gets
\top$ and $D \gets \top$, respectively, provided that $C$ and $D$
do not share any variables.
Most BUMG procedures support this splitting technique,
in particular, the provers we have used do.

\section{Range-Restricting Transformations}
\label{sec:transformations}
\label{sec:range-restriction}

Existing transformations to range-restricted form
follow~\citet{Manthey:Bry:SATCHMO:88} (or are variations of it).
The transformation 
can be defined by a procedure carrying out the following
steps on a given set $M$ of clauses.
\begin{description}
\item[(0) Initialization.] Initially, let $\crr(M) := M$.
\item[(1) Add a constant.]
Let $\dom$ be a `fresh' unary predicate symbol not in $\Sigma_P$,
and let~$c$ be some constant.
Extend $\crr(M)$ by the clause
\[
\dom(c) \gets \top.
\]
The constant $c$ can be `fresh' or belong to $\Sigma_f$.
\item[(2) Range-restriction.] 
For each clause $\calH \gets \calB$ in $\crr(M)$, let $\{
x_1,\ldots,x_k\}$ be the set of
variables occurring in $\calH$ but not in $\calB$. Replace 
$\calH \gets \calB$ by the clause 
\begin{equation*}
\calH \gets \calB \land \dom(x_1) \land \cdots \land \dom(x_k).
\end{equation*}
We refer to this clause as the clause \emph{corresponding} to $\calH
\gets \calB$.
\item[(3) Enumerate the Herbrand universe.] For each $n$-ary
$f \in \Sigma_f$, add the clauses:
\begin{displaymath}
  \dom(f(x_1,\ldots,x_n)) \gets \dom(x_1) \land \cdots \land \dom(x_n).
\end{displaymath}
\end{description}
The computed set $\crr(M)$ is the {\em classical
range-restricting transformation} of~$M$.
It is not difficult to see that $\crr(M)$ is indeed range-restricted for
any clause set~$M$.
The transformation is sound and complete, that is,
$M$ is satisfiable iff $\crr(M)$ is
satisfiable~\citep{Manthey:Bry:SATCHMO:88,BryYahya00}.
The size of $\crr(M)$ is linear in the size of $M$ and can
be computed in linear time.

Perhaps the easiest way to understand the transformation is to
imagine we use a BUMG method, for example, hyperresolution.  The idea is to build the model(s) during
the derivation.
The clause added in Step~(1) ensures that the domain of interpretation
given by the domain predicate $\dom$ is non-empty.
Step~(2) turns clauses into range-restricted clauses.
This is done by shielding the variables $\{x_1, \ldots, x_k \}$
in the head, that do not occur negatively, with the added negative
domain literals.
Clauses that are already range-restricted are unaffected by this step.
Step~(3) ensures that all elements of the Herbrand universe of the
(original) clause set are added to the domain via hyperresolution
inference steps.

As a consequence a clause set $M$ with at least one non-nullary function
symbols causes hyperresolution derivations to be unbounded
for $\crr(M)$, unless $M$ is unsatisfiable.
This is a negative aspect of the classical range-restricting transformation.
However, the method has been shown to be useful for 
(domain-)minimal model generation
when combined with other techniques~\citep{BryYahya00,BryTorge98}.
In particular, \cite{BryTorge98} use splitting and the $\delta^*$-rule
to generate domain minimal models.
In the present research we have evaluated the combination of blocking
techniques (introduced later in Section~\ref{sec:blocking}) with the
classical range-restricting transformation $\crr$.
This has shown promising empirical results as presented in
Section~\ref{sec:experiments}.

Next, we introduce a new transformation to range-restricted form.
Instead of enumerating the generally infinite Herbrand universe in
a bottom-up fashion, the intuition is that it generates terms only as needed.

The transformation involves extracting
the non-variable top-level terms in an atom.
Let $P(t_1,\ldots,t_n)$ be an atom and suppose $x_1,\ldots,x_n$
are fresh variables.
For all $i \in \{1,\ldots,n\}$  let 
$s_i = t_i$, if $t_i$ is a variable, and $s_i = x_i$, otherwise.
The atom $P(s_1,\ldots,s_n)$ is called the {\em term abstraction of $P(t_1,\ldots,t_n)$\/}.
Let the {\em abstraction substitution $\alpha$\/} be defined by
\begin{displaymath}
  \alpha  = \{ x_i \mapsto t_i \mid \text{$1 \leq i \leq n$ and $t_i$ is not a  variable}\}.
\end{displaymath}
Hence, $P(s_1,\ldots,s_n)\alpha = P(t_1,\ldots,t_n)$, that is,
$\alpha$ reverts the term abstraction.

The new {\em range-restricting transformation}, denoted by $\rr$,
of a clause set $M$ is the clause set obtained by carrying out the
following steps (explanations and an example are given afterwards):
\begin{description}
  \item[(0) Initialization.] Initially, let $\rr(M) := M$.
  \item[(1) Add a constant.] Same as Step~(1) in the definition of $\crr$.

\item[(2) Domain elements from clause bodies.] 
For each clause $\calH \gets \calB$ in $M$ and each atom
$P(t_1,\ldots,t_n)$ from $\calB$, let $P(s_1,\ldots,s_n)$ be the term abstraction of
$P(t_1,\ldots,t_n)$ and let $\alpha$ be the corresponding abstraction
substitution.
Extend $\rr(M)$ by the set
\begin{displaymath}
    \{ \dom(x_i)\alpha \gets P(s_1,\ldots,s_n)  \mid 1 \leq i \leq n\text{ and } x_i \mapsto t_i \in \alpha  \}.
\end{displaymath}

\item[(3) Range-restriction.] Same as Step~(2) in the definition of $\crr$. 

\item[(4) Domain elements from $\Sigma_P$.] For each $n$-ary
  %predicate symbol 
  $P$ in $\Sigma_p$, extend $\rr(M)$ by the set
  \begin{displaymath}
    \{ \dom(x_i) \gets P(x_1,\ldots,x_n)  \mid i \leq i \leq n\}.
  \end{displaymath}

\item[(5) Domain elements from $\Sigma_f$.] For each $n$-ary
  %function symbol 
  $f$ in $\Sigma_f$, extend $\rr(M)$ by the set
 \begin{displaymath}
   \{ \dom(x_i) \gets \dom(f(x_1,\ldots,x_n)) \mid 1 \leq i \leq n\}.
 \end{displaymath}
  \end{description}

The intuition of the transformation reveals itself if we think of 
what happens when using hyperresolution.
The idea is again to build model(s) during the derivation, but this
time terms are added to the domain only \emph{as necessary}.
Steps~(1) and (3) are the same as Steps~(1) and~(2) in the definition of
$\crr$. 
The clauses added in Step~(2) cause functional terms that occur
negatively in the clauses to be inserted into the domain.
Step~(4) ensures that positively occurring functional terms are added to the
domain, and Step~(5) ensures that the domain is closed under subterms.

To illustrate the steps of the transformation consider
the following clause.
\begin{align}
\tag{$\dagger$}
\ID q(x,\ID g(x,y)) \lor \ID r(y,z) & \gets \ID p( \ID a, \ID f(x,y), x)
\end{align}
It is added to $\rr(M)$ in Step~(0).
Suppose the clause added in Step~(1) is
\[
\dom(a).
\]
For Step~(2) the term abstraction of the body literal of
clause~($\dagger$) is $\ID p( x_1,
x_2, x)$ and the abstraction substitution is
  $\alpha = \{ x_1 \mapsto \ID a, x_2 \mapsto \ID f(x,y) \}$.
The clauses added in Step~(2) are the following:
\begin{align}
\notag
\dom(\ID a) & \gets \ID p(x_1,  x_2, x) 
\\
\tag{$\ddagger$}
\dom(\ID f(x,y)) & \gets \ID p(x_1, x_2, x)
\end{align}
Notice that among the four clauses we have so far the clauses~($\dagger$) and~($\ddagger$)
are not range-restricted.
They are however replaced by range-restricted clauses in
Step~(3), namely:
\begin{align}
\tag{$\dagger\dagger$}
\ID q(x,\ID g(x,y)) \lor \ID r(y,z) & \gets \ID p( \ID a, \ID f(x,y), x) \land \dom(z)\\
\notag
  \dom(\ID f(x,y)) & \gets \ID p(x_1, x_2, x) \land \dom(y).
\end{align}
Step~(4) generates clauses responsible for inserting
the terms that occur in the heads of clauses into the domain.
That is, for each $i \in \{1,2,3\}$ and each $j \in \{1,2\}$ these
clauses are added.
\begin{align*}
\dom(x_i) & \gets  \ID p(x_1, x_2, x_3) \\ 
\dom(x_j) & \gets  \ID q(x_1, x_2) \\
\dom(x_j) & \gets  \ID r(x_1, x_2)
\end{align*}
For instance, when a model assigns true to the instance $\ID q(\ID
a,\ID g(\ID a, \ID f(\ID a, \ID a)))$ of one of the head atoms of the
clause~($\dagger\dagger$), then $\dom(\ID a)$ and $\dom(\ID g(\ID a, \ID f(\ID a,
\ID a)))$ are also true. 
It is not necessary to insert the terms of the instance of the other
head atom into the domain. The reason is that it does not matter how these
(extra) terms are evaluated, or whether the atom is evaluated to true
or false in order to satisfy the disjunction.

The clauses added in Step~(4) alone are not sufficient, however.
For each term in the domain all its subterms have to be in the
domain, too. This is achieved with the clauses obtained in Step (5).
That is, for each $j \in \{1,2\}$ these clauses are added.
\begin{align*}
\dom(x_j) & \gets \dom(\ID f(x_1, x_2)) \\
\dom(x_j) & \gets \dom(\ID g(x_1, x_2))
\end{align*}

For the purposes of model generation, it is important to note that
one particular type of clause in the $\rr$ transformation
should not be treated as a normal clause.
For the equality predicate, Step~(4) produces the clauses
\begin{align*}
\tag{$\#$}
\dom(x) & \gets x \approx y
& \dom(y) & \gets x \approx y.
\end{align*}
Most theorem provers simplify these clauses to $\dom(x)$.
As a consequence this can lead to all negative
domain literals being resolved away and all clauses containing a
positive domain literal to be subsumed.
This means range-restriction is undone.
This is what happens in SPASS.

Since Step~(4) clauses really only need to be added for positively
occurring predicate symbols an easy solution 
involves replacing any positive occurrence of the equality predicate
by a predicate symbol $\myequal$ (say), which is fresh in the signature,
and adding the clauses
\begin{align*}
\dom(x) & \gets \myequal(x,y) 
& \dom(y) & \gets \myequal(x,y) 
\end{align*}
in Step (4) rather than~($\#$).
In addition, the clause set needs to be extended by this definition of
$\myequal$.
\begin{align*}
x \approx y & \gets \myequal(x,y)
\end{align*}
This solution has the intended effect of adding terms occurring in
positive equality literals to the domain, and prevents other inferences
or reductions on $\myequal$. 
It is not difficult to prove that E-satisfiability is preserved in both
directions. We will implicitly use this fact in the proofs below.

\begin{proposition}[Completeness of range-restriction]%
 \label{prop:rr-completeness} Let $M$ be any clause set.  
 If~$\rr(M)$ is satisfiable then $M$ is satisfiable.
  \end{proposition}
\begin{proof} 
Suppose $\rr(M)$ is satisfiable. Let $I_\rr$ be a Herbrand model of
$\rr(M)$. We define a quasi-Herbrand interpretation $I$
and show that it is a model of $M$. 

First, the domain of $I$ is defined as the set
$|I| = \{ t \mid I_\rr \models \dom(t) \}$.

Now, to define a total interpretation for the
function symbols, we map each $n$-ary function symbol $f$ in 
$\Sigma_f$ to the function $f^I: |I| \times \cdots \times |I| \mapsto |I|$, where, for all
$d_1,\ldots,d_n\in |I|$, 
\begin{displaymath}
  f^I(d_1,\ldots,d_n) := 
    \begin{cases}
      f(d_1,\ldots,d_n) & \text{if $f(d_1,\ldots,d_n) \in |I|$, and} \\
      c & \text{otherwise.}
    \end{cases}
\end{displaymath}
Here, the constant $c$ is the one mentioned in Step (1) of
the transformation. (It is clear that $|I|$ contains $c$.)

Notice that due to Step (5) the domain $|I|$ must contain for each
term all its subterms. An easy consequence is that all terms in $|I|$
are evaluated as themselves, exactly as in Herbrand interpretations.
Each other (ground) term is evaluated as some other term from 
$|I|$.
For instance, if $|I| = \{ c, f(c)\}$ then
  $I(f(g(c))) = f^I(I(g(c))) = f^I(g^I(c)) = f^I(c) = f(c)$, since $g(c)
\not\in |I|$ and by the definition of $g^I$.
We see that $f$
is indeed mapped to a total function over the domain $|I|$, as
required.

Regarding the interpretation of the predicate symbols in $I$,  define for
every $n$-ary predicate symbol $P$ in $\Sigma_P$ and for all
$d_1,\ldots,d_n\in |I|$:
\begin{equation}
  \label{eq:def-P-J}
  P(d_1,\ldots,d_n) \in I\text{ iff } P(d_1,\ldots,d_n) \in {I_\rr}\enspace.
\end{equation}
That is, the interpretation of the predicate symbols in $I$ is the
same as in $I_\rr$ under the restriction of the domain to $|I| \subseteq |I_\rr|$.

It remains to show that $I$ is a model of $M$. It suffices to pick a
clause $\calH \gets \calB$ from~$M$ arbitrarily and show that $I$
satisfies this clause. We do this by assuming that $I$ does not
satisfy $\calH \gets \calB$ and deriving a contradiction. 

That $I$ does not satisfy $\calH \gets \calB$ means there is a
valuation $\mu$ such that $(I,\mu)
\models \calB$ but $(I,\mu) \not\models \calH$.
As usual, a valuation is a (total) mapping from the variables to the
domain under consideration.

Because the domain $|I|$ consists of (ground) terms, the valuation
$\mu$ can be seen as a substitution. Thus, $\calB\mu$ is a set of ground
atoms, and $\calB\mu \subseteq I$ may or may not hold.
We show
next that if $(I,\mu) \models \calB$, as given, then $\calB\mu \subseteq I$. In other
words, the body is
satisfied in~$I$ because~$|I|$ contains all body atoms $\calB\mu$,
but not for the reason that~$I$ assigns true to some body atom $B$ with some
argument term evaluated to $c$, and that atom being contained in 
$I$.
An example for the latter case is $|I| = \{\ID c\}$, $\calB =
  \ID P(x)$, $I = \{ \ID P(\ID c) \}$ and $\mu = \{ x \mapsto \ID a \}$. Although we
  have $(I,\mu) \models P(x)$, in essence because $\ID{a}^I = \ID c$, it does {\em not\/}
  hold that $\ID P(\ID a) \in I$.
The relevance of this result is that it allows syntactically based
reasoning further below to show that $I$ is a model of $M$.

To show $\calB\mu \subseteq I$ it suffices to choose any body literal
$P(t_1,\ldots,t_n)$ from $\calB$ arbitrarily and show $P(t_1,\ldots,t_n)\mu \in
I$.  
Now from $(I,\mu) \models \calB$ it follows that $(I,\mu) \models P(t_1,\ldots,t_n)$.
Reading~$\mu$ as a ground substitution this
means $P(I(t_1\mu),\ldots,I(t_n\mu)) \in I$. 
Using the equivalence~\eqref{eq:def-P-J} it follows that\linebreak
$P(I(t_1\mu),\ldots,I(t_n\mu)) \in I_\rr$. To show $P(t_1,\ldots,t_n)\mu \in I$, as
required, it thus suffices to show $I(t_i\mu) = t_i\mu$, because
$P(t_1,\ldots,t_n)\mu \in I$ follows from\linebreak $P(I(t_1\mu),\ldots,I(t_n\mu)) \in
I_\rr$ and equivalence~\eqref{eq:def-P-J}.

Thus, let us show $I(t_i\mu) = t_i\mu$.
By the definition of the interpretation function~$\cdot^I$ 
it is enough to show $t_i\mu \in |I|$ (as said above, terms from $|I|$
are evaluated to themselves).  
If $t_i$ is a variable then $t_i\mu \in |I|$ follows from the
fact that $\mu$ was chosen as a substitution into~$|I|$.  Assume
now that $t_i$ is not a variable and let $P(s_1,\ldots,s_n)$ be the term
abstraction of $P(t_1,\ldots,t_n)$ and~$\alpha$ its abstraction substitution.
By transformation Step~(2), $\rr(M)$ includes the clause
\begin{equation}
  \dom(x_i)\alpha \gets P(s_1,\ldots,s_n),
\label{eq:dom-p}
\end{equation}
where $\{ x_i \mapsto t_i \} \in \alpha$.
By the definition of an abstraction, for all $j \in \{ 1,\ldots,n\}$, $s_j$ is a fresh
variable whenever $t_j$ is not a variable. 

Recall from above that $P(I(t_1\mu),\ldots,I(t_n\mu)) \in I_\rr$. We are going to
show now that with clause~\eqref{eq:dom-p}  this entails $\dom(t_i\mu)$.
By the construction of $|I|$ this suffices to prove $t_i\mu \in |I|$, as
desired.

Consider the substitution 
\begin{displaymath}
  \mu' = \mu \{ x_j \mapsto I(t_j\mu) \mid x_j \mapsto t_j \in \alpha
\}.
\end{displaymath}
It agrees with $\mu$ (in particular) when $t_j$ is a variable  and
otherwise maps the variable $x_j$ to $I(t_j\mu)$.

When $t_j$
is a variable then let $s_j = t_j$ be the definition of an abstraction. This
means $s_j\mu' = s_j\mu = t_j\mu = I(t_j\mu)$ (the latter identity holds, again,
because $\mu$ is a substitution into $|I|$ and elements from $|I|$
evaluate to themselves). When $t_j$ is not a variable then $s_j$
is the variable~$x_j$. By construction of $\mu'$ we have $s_j\mu' =
x_j\mu' = I(t_j\mu)$. Hence, in both cases $s_j\mu' = I(t_j\mu)$.

Applying the substitution $\mu'$ to the clause~\eqref{eq:dom-p} yields 
\begin{displaymath}
  \dom(x_i)\alpha\mu' \gets P(s_1,\ldots,s_n)\mu'\enspace.
\end{displaymath}
With the identities $s_j\mu' = I(t_j\mu)$, the identities
$\dom(x_i)\alpha\mu' =\dom(t_i)\mu'$ and the fact that
$P(I(t_1\mu),\ldots,I(t_n\mu))
\in I_\rr$ it follows that $\dom(t_i)\mu' \in I_\rr$. The substitution $\mu$ and~$\mu'$
differ in their domains only on the fresh variables $x_1,\ldots,x_n$. 
Therefore $\dom(t_i)\mu' = \dom(t_i)\mu$ and $\dom(t_i)\mu \in I_\rr$
follows, as desired.

This was the last subgoal to be proven to establish
$P(t_1,\ldots,t_n)\mu \in I$, which, in turn,  remained to be
shown to complete the proof that $\calB\mu \subseteq I$. 

The next step in the proof is to show that the clause body
of the clause in $\rr(M)$ corresponding to $\calH \gets \calB$ is
satisfied by $I_\rr$. That clause is the range-restricted version of 
the clause $\calH \gets \calB$ in $M$. According to 
Step (3) of the transformation it has the form 
\begin{equation}
\calH \gets \calB \land \dom(x_1) \land \cdots \land \dom(x_k)
\label{eq:H-B-rr}
\end{equation}
for some variables $x_1,\ldots,x_k$; those occurring in $\calH$ but not
in $\calB$. 

From $\calB\mu \subseteq I$ as derived above and 
equivalence~\eqref{eq:def-P-J} it follows that $\calB\mu \subseteq I_\rr$.
Recall that~$\mu$ is a valuation mapping into the domain $|I|$. Reading it as a
substitution gives $x_j\mu \in |I|$, for all $j \in \{ 1,\ldots,k\}$.  From the
construction of $|I|$ it follows that $\dom(x_j\mu) \in I_\rr$.  Together
with $\calB\mu \subseteq I_\rr$ and the fact that $I_\rr$ is a model of
$\rr(M)$, and hence of clause~\eqref{eq:H-B-rr}, it follows
that $I_\rr$ satisfies $\calH\mu$. This means $A\mu \in I_\rr$ for some
head atom $A$ in $\calH$. 

The atom $A$ is of the form $Q(s_1,\ldots,s_m)$
for some $m$-ary predicate symbol $Q$ and terms $s_1,\ldots,s_m$.
By Step (4) of the transformation, $\rr(M)$ includes,
for all $i \in \{ 1,\ldots,m\}$ the clause
\begin{equation}
    \dom(x_i) \gets Q(x_1,\ldots,x_m).
\label{eq:dom-clause-rr}
\end{equation}
Again by reading $\mu$ as a substitution, because $I_\rr$ is a model of
$\rr(M)$, and hence of clause~\eqref{eq:dom-clause-rr}, and by the
identities $Q(s_1\mu,\ldots,s_m\mu) = Q(s_1,\ldots,s_m)\mu = A\mu \in I_\rr$
we conclude $\dom(s_i\mu) \in I_\rr$, for all $i \in \{ 1,\ldots,m\}$.  By
construction of $|I|$ we have that $s_i\mu \in |I|$. By
equivalence~\eqref{eq:def-P-J} it follows that $Q(s_1\mu,\ldots,s_m\mu) \in I$.

Recall that $Q(s_1,\ldots,s_m)$ is a head atom of the clause
\eqref{eq:H-B-rr} and hence a head atom of the clause $\calH \gets
\calB$.  Further recall that $s_i\mu \in |I|$ entails that $s_i\mu$ is
evaluated to itself in $I$. Together with $Q(s_1\mu,\ldots,s_m\mu) \in I$
this means $(I,\mu) \models Q(s_1,\ldots,s_m)$. This is a 
contradiction to $(I,\mu) \not\models \calH$ as concluded above.
The proof is complete. \qed
  \end{proof}

The proof actually gives a characterization of the models associated
with a satisfiable clause set $\rr(M)$.

\begin{corollary}[Completeness of range-restriction wrt.\
  E-interpretations]
 \label{cor:rr-E-completeness} Let~$M$ be any clause set.  
 If $\rr(M)$ is E-satisfiable then $M$ is E-satisfiable.
  \end{corollary}
  \begin{proof}
We prove the contra-positive statement. Thus assume $M$ is
E-unsatisfiable. Equivalently,  $M \cup \EAX(\Sigma_P \cup \Sigma_f)$ is unsatisfiable.
By Proposition~\ref{prop:rr-completeness},
$\rr(M \cup \EAX(\Sigma_P \cup \Sigma_f))$ is unsatisfiable.
Observe Steps~(2) and~(3), which are the only ones that
apply directly to the given clauses, have no effect on the equality
axioms $\EAX(\Sigma_P \cup \Sigma_f)$, except for the reflexivity axiom $x\approx x$,
which is replaced by $x\approx x \gets \dom(x)$.
The transformed set $\rr(M \cup \EAX(\Sigma_P \cup \Sigma_f))$ coincides
with
\begin{displaymath}
\rr(M) \cup (\EAX(\Sigma_P \cup \Sigma_f) \setminus \{ x\approx x \}) \cup \{ x\approx x \gets \dom(x) \}.
\end{displaymath}
Adding back the reflexivity axiom trivially preserves
unsatisfiability, that is, with $\rr(M \cup \EAX(\Sigma_P \cup \Sigma_f))$ being
unsatisfiable, so is
\begin{displaymath}
\rr(M) \cup \EAX(\Sigma_P \cup \Sigma_f) \cup \{ x\approx x \gets \dom(x) \}.
\end{displaymath}
The clause $x\approx x \gets \dom(x)$ can be deleted because it is subsumed by the clause $x\approx x
\in \EAX(\Sigma_P \cup \Sigma_f)$. Hence,
\begin{displaymath}
\rr(M) \cup \EAX(\Sigma_P \cup \Sigma_f)
\end{displaymath}
is unsatisfiable, and so $\rr(M)$ is E-unsatisfiable. \qed
\end{proof}
We emphasize that we do not propose to actually use the equality
axioms in conjunction with a theorem prover (though they can of course). They serve merely as a
theoretical tool to prove completeness of the transformation.

\begin{proposition}
Let~$M$ be any clause set. Then
\begin{enumerate}[(i)]
\item
The size of $\rr(M)$ is bounded by a linear function
in the size of $M$.
\item
$\rr(M)$ can be computed in quadratic time.
\item
$\rr(M)$ is range-restricted.
\end{enumerate}
\label{result_rr_Horn_preservation}
\end{proposition}
By carefully modifying the definition of $\rr$ and at the expense
of some duplication it is possible to compute the reduction in linear
time.

Proposition~\ref{result_rr_Horn_preservation}~(iii) confirms that every
clauses produced by the $\rr$ transformation is range-restricted.

Let us consider another example to get a better understanding of the
$\rr$ transformation.
\begin{align}
\tag{$*$}
\ID r(x) & \gets \ID q(x) \land \ID p(\ID f(x)).
\end{align}
Applying Steps~(2) and~(3) of the $\rr$ transformation gives us the clause
\begin{align*}
\dom(\ID f(x)) & \gets \dom(x) \land \ID p(y).
\end{align*}
This clause is splittable into 
\begin{align*}
\dom(\ID f(x)) \gets \dom(x)
\quad \text{and} \quad \bot \gets \ID p(y).
\end{align*}
The first split component clause is an
example of an `enumerate the Herbrand universe' clause from the
$\crr$ transformation (Step~(2) in the definition of $\crr$).
Such clauses are unpleasant because they cause the entire Herbrand
universe to be enumerated with BUMG approaches.

Before describing a solution let us analyze the problem further.
The main rationale of our $\rr$ transformation is to constrain the
generation of domain elements and limit the number of inference steps.
The general form of clauses produced by Step~(2), followed by Step~(3),
is the following, where $\overline y \subseteq \overline x$, $\overline
x \subseteq \overline y \cup \overline z$ and
$\overline u \subseteq \overline z$.
\begin{align*}
\dom(f(\overline x)) & \gets \dom(y_1) \land \ldots \land \dom(y_n)
\land P(\overline z) \\
\dom(f(\overline u)) & \gets P(\overline z)
\end{align*}
Clauses of the first form are often splittable (as in the example
above), and can produce clauses of the unwanted form
\[
\dom(f(\overline y)) \gets \dom(y_1) \land \ldots \land \dom(y_n).
\]
Suppose therefore that splitting of any clause is forbidden when this
splits the negative part of the clause (neither SPASS nor a hypertableaux prover do this anyway).
Although the two types of clauses above both \emph{do} reduce the
number of terms created, compared to the classical range-restricting
transformation, the constraining effect of the first type of clauses is
slightly limited.
Terms $f(\overline s)$ are not generated, only when no fact $P(\overline
t)$ is present or has been derived.
When a clause~$P(\overline t)$ is present, or as soon as such a clause is
derived (for \emph{any} ground terms $\overline t$), then terms are freely
generated from terms already in the domain with~$f$ as the top symbol.

Here is an example of a clause set for which the derivation is infinite
on the $\rr$~transformation.
(The example is an extension of the example above with the clause
$\ID p(\ID b) \gets \top$.)
\begin{align*}
\ID p(\ID b) & \gets \top &
\ID r(x) & \gets \ID q(x) \land \ID p(\ID f(x)) &&
\end{align*}
Notice the derivation is infinite on the classical
range-restricting transformation as well, due to the generated clauses 
$\ID{dom}(\ID b) \gets \top$ and $\ID{dom}(f(x)) \gets \ID{dom}(x)$.

The second type of clauses, $\dom(f(\overline u)) \gets P(\overline
z)$, are less problematic.
Here is a concrete example.
For $\bot \gets \ID r(x,\ID f(x))$,
Step (2) produces the clause
\begin{align*}
\dom(\ID f(x)) \gets \ID r(x,y).
\end{align*}
Although this clause, and the general form, still causes larger terms
to be built with hyperresolution type inferences, the constraining
effect is larger.

In the next two sections we discuss ways of improving range-restricting
transformations further.

\section{Shifting Transformation}
\label{sec:shifting}

The clauses introduced in Step (2) of the new $\rr$~transformation to
range-restricted form use abstraction and insert (possibly a large number of) instantiations of terms occurring in
the clause bodies into the domain.
These are sometimes unnecessary and can lead to non-termination of BUMG
procedures.

The {\em shifting\/} transformation introduced next can address this problem.
It consists of two sub-transformations, {\em basic shifting\/} and
{\em partial flattening\/}.

If $A$ is an atom $P(t_1,\ldots,t_n)$ then let $\NOT{A}$
denote the atom $\NOT{P}(t_1,\ldots,t_n)$, where $\NOT{P}$ is a fresh predicate
symbol which is uniquely associated with the predicate symbol $P$.
If $P$ is the equality symbol $\approx$ we write $\NOT{P}$ as
$\not\approx$ and use infix notation.

Now, the {\em basic shifting transformation} of a clause set $M$
is the clause set $\bs(M)$ 
obtained from~$M$ by carrying out the following steps.
  \begin{description}
  \item[(0) Initialization.] Initially, let $\bs(M) := M$.

  \item[(1) Shifting deep atoms.] 
Replace each clause in $\bs(M)$ of the form $\calH \gets B_1 \land \cdots \land B_m
\land \calB$, where each atom $B_1,\ldots,B_m$ contains at least one proper
functional term and $\calB$ contains no proper functional term, by the
clause 
\begin{displaymath}
  \calH \lor \NOT{B_1} \lor \cdots \lor \NOT{B_m} \gets \calB.
\end{displaymath}
Each of the atoms $B_1,\ldots,B_m$ is called a {\em shifted atom\/}.

\item[(2) Shifted atom consistency.]  Extend $\bs(M)$ by the clause set 
  \begin{multline*}
    \{ \bot \gets P(x_1,\ldots,x_n) \land \NOT{P}(x_1,\ldots,x_n) \mid \\
\text{$P$ is the $n$-ary predicate symbol of a shifted atom}\}.
  \end{multline*}
\end{description}

Notice that we do {\em not\/} add  clauses complementary to the `shifted
atoms consistency' clauses, that is, $P(x_1,\ldots,x_n) \lor
\NOT{P}(x_1,\ldots,x_n) \gets \top$.
They could be included but are superfluous.

Let us continue the example given at
the end of the previous section.
We can use basic shifting to move negative
occurrences of functional terms into heads.
In the example, clause~$(*)$ is replaced by
\begin{align*}
\ID r(x) \lor \NOT{\ID p}(\ID f(x)) & \gets \ID q(x) \tag{$**$} \\
\bot & \gets \NOT{\ID p}(x) \land \ID p(x) \notag \\
\dom(x) & \gets \ID r(x) \notag &
\dom(x) & \gets \NOT{\ID p}(x) \notag &
\end{align*}
Even in the presence of an additional clause, say, 
$\ID q(x) \gets \top$, which leads to the clauses
\begin{align*}
\dom(\ID a) & \gets \top & \ID q(x) \gets \dom(x), &&
\end{align*}
termination of BUMG can be achieved.

For instance, in a hyperresolution-like mode of operation and with splitting
enabled, the SPASS prover~\citep{WeidenbachSchmidtEtAl07,WeidenbachSPASS35} splits
the derived clause $\ID r(\ID a) \lor \NOT{\ID p}(\ID f(\ID a))$,
considers the case with the smaller literal  $\ID r(\ID a)$ first
{\em and terminates with a model\/}. This is because a finite completion (model) is
found without considering the case of the bigger literal $\NOT{\ID
  p}(\ID f(\ID a))$, 
which would have added the deeper term~$\ID f(\ID a)$ to the domain.
The same behaviour can be achieved, for example, with
the KRHyper BUMG prover, a hypertableaux theorem prover~\citep{Wernhard:KRHyper:CADE-19-WS:2003}.

As can be seen in the example, the basic shifting transformation
trades the generation of new domain elements for a smaller clause
body by removing literals from it.
Of course, a smaller clause body affects the search space,
as then the clause can be used as a premise more often.
To (partially) avoid this effect,
we propose an additional transformation to be performed prior to the
basic shifting transformation.

For a clause set $M$, the {\em partial flattening transformation}
is the clause set $\pf(M)$ obtained by applying
the following steps.
  \begin{description}
  \item[(0) Initialization.] Initially, let $\pf(M) := M$.
  \item[(1) Reflexivity.] Extend $\pf(M)$ by the unit clause $x\approx x \gets \top{}$.
\item[(2) Partial flattening.] 
For each clause $\calH \gets \calB$ in $\pf(M)$,
let $t_1,\ldots,t_n$ be all top-level terms occurring in the
non-equational literals
in the body $\calB$  that are proper functional terms, 
for some $n \geq 0$. Let $x_1,...,x_n$ be fresh variables.
Replace the clause $\calH \gets \calB[t_1,\ldots,t_n]$ by the clause 
\begin{displaymath}
\calH \gets \calB[x_1,\ldots,x_n] \land t_1\approx x_1 \land \cdots \land t_n \approx x_n.
\end{displaymath}
\end{description}

It should be noted that the equality symbol $\approx$ need not
be interpreted as equality, but could. (Un-)satisfiability (and logical
equivalence) is preserved even when reading it just as
`unifiability'. This can be achieved by the clause $x\approx x \gets
\top$.
One should however note that the reflexivity clause is not compatible
with introducing the $\myequal$ predicate, so this might not always be a
possibility.
(In our implementation, for this reason the reflexivity clause is not
added.)

In our running example, applying the transformations $\pf$, $\bs$ and $\rr$,
in this order, yields the following clauses (among other clauses,
which are omitted because they are not relevant to the current
discussion).
\begin{align*}
\ID r(x) \lor \ID f(x) \not\approx u & \gets \ID q(x) \land \ID p(u) &
\dom(x) & \gets x \not\approx y &
\dom(x) & \gets \ID r(x)\\
\bot & \gets x \not\approx y \land x \approx y &
\dom(y) & \gets x \not\approx y
\end{align*}
Observe that the first clause is more restricted than the clause $(**)$
above because of the additional body literal $\ID p(u)$.

The reason for not extracting constants during partial flattening is
that adding them to the domain does not cause non-termination of BUMG
methods. It is preferable to leave them in place in the body
literals because they have a stronger constraining effect than the
variables introduced otherwise.

Extracting top-level terms from equations has no effect at all.
Consider the unit clause $\bot \gets f(a) \approx b$, and its partial
flattening $\bot \gets x \approx b \land f(a) \approx x$.
Applying basic shifting yields $f(a) \not\approx x \gets x \approx
b$, and, hyperresolution with $x\approx x \gets \top $ gives $f(a)
\not\approx b \gets \top$.  This is the same result as obtained by the
transformations as defined.  This explains why top-level terms of
equational literals are excluded from the definition. (One
could consider using `standard' flattening, that is, recursively extracting
terms, but this does not lead to any improvements over the defined transformations.)

Finally, we combine basic shifting and partial flattening to give 
the {\em shifting transformation}, formally defined by $\sh := \pf
\circ \bs$, that is, $\sh(M) = \bs(\pf(M))$, for any clause set~$M$.

\begin{proposition}[Completeness of shifting]
\label{prop:result_completeness_shifting}
Let $M$ be any clause set.  
 If $\sh(M)$ is satisfiable then $M$ is satisfiable.
\end{proposition}
\begin{proof}
Not difficult, since $\bs$ (basic shifting) can be seen to be a structural
transformation and~$\pf$ (partial flatting) is a form of term abstraction.
\qed
\end{proof}

\begin{corollary}[Completeness of shifting wrt.\ E-interpretations]
\label{cor:result_completeness_shifting}
Let $M$ be any clause set.  
 If $\sh(M)$ is E-satisfiable then $M$ is E-satisfiable.
\end{corollary}
\begin{proof}
Using the same line of argument as in the
proof of Corollary~\ref{cor:rr-E-completeness}, proving preservation of
E-satisfiability can be reduced to proving preservation of
satisfiability by means of the equality axioms (observe that the
shifting transformation does not modify the equality axioms). \qed
\end{proof}

\section{Blocking}
\label{sec:blocking}

The final transformation introduced in this paper is called \emph{blocking} and 
provides a mechanism for detecting
recurrence in the derived models. 
The blocking transformation is designed to realize a `loop check'
for the construction of a domain, by capitalizing on available,
powerful equality reasoning technology and redundancy criteria from
saturation-based theorem proving.
To be suitable, a resolution-based prover, for instance, should support
\emph{hyperresolution-style inference, strong equality inference (for
example, superposition or ordered rewriting), splitting,
and the possibility to search for split-off equations first and standard
redundancy elimination techniques}.

The basic idea behind blocking is to add clauses that cause a case
analysis of the form $s \approx t$ versus $s \not\approx t$, for 
(ground) terms $s$ and $t$.
Although such a case analysis obviously leads to a bigger search space,
it provides a powerful technique to detect finite models with a BUMG prover.
This is because in the case that $s \approx t$ is assumed, this new equation may lead
to rewriting of otherwise infinitely many terms into one single term.
To make this possible, the prover must support the above features,
including notably splitting.
Among resolution theorem provers splitting has become standard.
Splitting was first available in the saturation-based
prover SPASS~\citep{WeidenbachSchmidtEtAl07,WeidenbachSPASS35}, but is now also part of
VAMPIRE~\citep{RiazanovVoronkov02} and E~\citep{Schulz13}. 
Splitting is an integral part of the hypertableau prover
E-KRHyper~\citep{Baumgartner:Furbach:Pelzer:HyperTableauxEquality:CADE:2007,Pelzer:Wernhard:E-KRHyper:CADE:2007}.

Blocking has the same goal as the \emph{unsound theorem proving} technique introduced
first in~\cite{Lynch:UnsoundTheoremProving:CSL:2004}. 
Instances of unsound theorem proving exemplified
in~\cite{Lynch:UnsoundTheoremProving:CSL:2004} include replacing a clause by one that
subsumes it, and by adding equations for joining equivalence classes in the abstract
congruence closure framework. Unsound theorem proving has been incorporated later in
DPLLT-based theorem proving~\cite{Bonacina:etal:SpeculativeInferences:JAR:2011}.

In the following we introduce four different, but closely related, blocking
transformations, called
\emph{subterm domain blocking}, \emph{subterm predicate blocking},
\emph{unrestricted domain blocking} and  \emph{unrestricted predicate
blocking}.
Subterm domain blocking was introduced in the short version of this
paper under the name blocking \citep{BaumgartnerSchmidt06}.
Subterm predicate blocking is inspired by and related to the blocking
technique described in \cite{HustadtSchmidt99b}.
Unrestricted domain blocking is the first-order version of the
unrestricted blocking rule introduced in \citet{SchmidtTishkovsky07b}
and used for developing terminating tableau calculi for logics with the effective
finite model property in~\cite{SchmidtTishkovsky08b,SchmidtTishkovsky11a}.

\subsection{Subterm Domain Blocking}

By definition, the {\em subterm domain blocking transformation}
of a clause set $M$
is the clause set $\blsd(M)$ obtained from $M$ by carrying
out the following steps.
  \begin{description}
  \item[(0) Initialization.] Initially, let $\blsd(M) := M$.
  \item[(1) Axioms describing the subterm relationship.]
Let $\subterm$ be a `fresh' binary predicate symbol not in $\Sigma_P$.
 Extend $\blsd(M)$ by
 \begin{align*}
   \subterm(x,x) & \gets \dom(x)
 \end{align*}
and, for every $n$-ary function symbol $f \in \Sigma_f$ and all
$i\in \{1,\ldots,n\}$, add the clauses
\begin{align*}
   \subterm(x,f(x_1,\ldots,x_n)) & \gets \subterm(x,x_i) \land \dom(x) \land \dom(f(x_1,\ldots,x_n)).
 \end{align*}

\item[(2) Subterm equality case analysis.] 
 Extend $\blsd(M)$ by these clauses.
 \begin{align*}
   x\approx y \lor x \not\approx y & \gets \subterm(x,y) \\
   & \gets x\approx y \land x \not\approx y
 \end{align*}
\end{description}
The subterm domain blocking transformation allows to contemplate
whether two domain elements that are in a subterm relationship
should be identified and merged, or not.

This blocking transformation preserves
range-restrictedness. In fact, because the $\dom$
predicate symbol is mentioned in the definition, the blocking
transformation can be applied meaningfully only in combination
with range-restricting transformations.

Reading $\subterm(s,t)$ as `$s$ is a subterm of $t$', Step~(1) in
the blocking transformation might seem overly involved, because an apparently
simpler specification  of the subterm relationship for the terms of
the signature $\Sigma_f$ can be given.
Namely:
\begin{align*}
   \subterm(x,x) & \gets \dom(x) &
   \subterm(x,f(x_1,x_2\ldots,x_n)) & \gets \subterm(x,x_i)
 \end{align*}
for every $n$-ary function symbol $f \in \Sigma_f$ and all $i\in\{1,\ldots,n\}$.
This clause set is range-restricted. 
Yet, this specification is not suitable for our purposes.
The problem is that the second clause introduces proper functional
terms.

For example, for a given constant $\ID a$ and a unary function symbol
$\ID f$, when just $\dom(\ID a)$ alone has been derived, a BUMG
procedure derives an infinite sequence of clauses:
\begin{displaymath}
\subterm(\ID a, \ID a), \subterm(\ID a, \ID f(\ID a)), \subterm(\ID a,
\ID f(\ID f(\ID a))), \ldots.
\end{displaymath}
This does not happen with the specification in Step~(1).
It ensures that conclusions of
BUMG inferences involving $\subterm$ are about terms currently in
the domain, and the domain is always finite.

To justify the clauses added in Step~(2) we continue this example and
suppose an interpretation that contains $\dom(\ID a)$ and $\dom(\ID
f(\ID a))$.  These might have been derived earlier in
the run of a BUMG prover.
Then, from the clauses added by blocking, the (necessarily ground) disjunction 
\begin{displaymath}
  \ID f(\ID a) \approx \ID a \lor \ID f(\ID a) \not\approx \ID a \gets
\top
\end{displaymath}
is derivable.

Now, it is important to use a BUMG prover with support for
splitting and to equip it with an appropriate search strategy.
In particular, when deriving a disjunction such as the one above, the
$\approx$-literal should be split off and the clause set obtained in
this case should be searched {\em first\/}.  The reason is that the
(ground) equation $\ID f(\ID a) \approx \ID a$ thereby obtained can
then be used for simplification and redundancy testing purposes.
For example, should $\dom(\ID f(\ID f(\ID a)))$ be derivable now (in the
current branch), 
then any prover based a modern, saturation-based theory of
equality reasoning is able to prove it redundant from $\ID f(\ID
a) \approx \ID a$ and $\dom(\ID a)$. Consequently, the domain is {\em not\/}
be extended {\em explicitly\/}. The information that $\dom(\ID
f(\ID f(\ID a)))$ is in the domain is however implicit via the
theory of equality.

\subsection{Subterm Predicate Blocking}
\label{sec:subterm-predicate-blocking}
Subterm domain blocking defined in the previous
section applies blocking to \emph{domain terms} where one is a
proper subterm of the other.
The idea of the \emph{subterm (unary) predicate blocking transformation} is
similar, but it merges only the (sub)terms in the extension of \emph{unary}
predicate symbols different to $\dom$ in the current interpretation.

Subterm predicate blocking is
defined as follows:

  \begin{description}
  \item[(0) Initialization.] Initially, let $\blsp(M) := M$.
  \item[(1) Axioms describing the subterm relationship.]
    Same as Step~(1) in the definition of $\blsd$.
\item[(2) Subterm equality case analysis.] 
 Extend $\blsp(M)$ by these clauses, for each unary predicate symbol
 $p \in \Sigma_P$. (Recall that $\Sigma_P$ does not contain $\dom$.)
 \begin{equation*}
   x\approx y \lor x \not\approx y  \gets \subterm(x,y) \land p(x) \land p(y)
 \end{equation*}
Finally, add the clause 
\[
{}\gets x\approx y \land x \not\approx y
\]
to $\blsp(M)$.
\end{description}

Observe that the only difference between this transformation and 
the subterm domain blocking transformation lies in Step (2). 
The clauses $x\approx y \lor x \not\approx y \gets \subterm(x,y) \land p(x) \land
p(y)$ added here are obviously more restrictive than their counterpart
$x\approx y \lor x \not\approx y \gets \subterm(x,y)$ in the definition
of the subterm domain blocking transformation \blsd.

That subterm predicate blocking is strictly more
restrictive can be seen from the following example, which also
helps to explain the rationale behind this transformation.
\begin{align*}
\ID p( \ID a)  & \leftarrow &
\ID q( \ID f(x))  & \leftarrow \ID p(x)
\end{align*}
Any BUMG prover terminates on the transformed clause set and
returns the model 
\begin{equation*}
  \{ \dom(\ID a),\ \dom(\ID f(\ID a)),\ \ID p( \ID a),\ \ID q( \ID
  f(\ID a)),\ \subterm(\ID a,\ID f(\ID a)) \}.
\end{equation*}
Notice that the subterm predicate blocking transformation includes the
clauses
\begin{align*}
 x\approx y \lor x \not\approx y & \gets \subterm(x,y) \land \ID p(x) \land \ID p(y) \\
 x\approx y \lor x \not\approx y & \gets \subterm(x,y) \land \ID q(x) \land \ID q(y).
\end{align*}
These are however only applicable for $\subterm(\ID s,\ID s)$, $\ID
p(s)$ and $\ID q(s)$ which only lead to redundant BUMG inferences.
The motivation behind these clauses is to block two $\ID p$-literals (say)
only when there are two literals $\ID p(s)$ and $\ID p(t)$ where~$s$ is a subterm
of~$t$.
Conversely, if no such loop comes up, as in the example above, there is
no reason for blocking. By contrast, the subterm domain blocking
transformation $\blsd$ with its clause 
$x\approx y \lor x \not\approx y  \gets \subterm(x,y)$ would be applicable even for distinct terms, leading to the
(unnecessary) split into the cases $\ID a \approx \ID f(\ID a)$ and $\ID a \not\approx \ID f(\ID a)$.

From a more general perspective, the $\blsp$ transformation is
motivated by the application to description logic knowledge bases~\citep{HustadtSchmidt99b,BaaderSattler01}.
Often, such knowledge bases do not contain cyclic definitions, or
only few definitions are cyclic.
The subterm predicate transformation aims to apply blocking only to
concepts (unary predicates) with cyclic definitions.
Below, in Section~\ref{section_DL_example}, we discuss a description
logic example to highlight the differences between the various
blocking transformations.

\subsection{Unrestricted Domain Blocking}

The two previous `subterm' variants of the blocking transformation
allow to speculatively identify terms \emph{and their subterms}. 
The `unrestricted' variants introduced next differ from both by
allowing speculative identifications of \emph{any} two terms. 

For the `domain' variant, called \emph{unrestricted domain blocking
transformation}, the definition is as follows.
  \begin{description}
  \item[(0) Initialization.] Initially, let $\blud(M) := M$.
  \item[(1) Domain elements equality case analysis.] 
 Extend $\blud(M)$ by these clauses.
 \begin{align*}
   x\approx y \lor x \not\approx y & \gets \dom(x) \land \dom(y) \\
   & \gets x\approx y \land x \not\approx y
 \end{align*}
\end{description}

There is a clear trade-off between this transformation and the
subterm domain blocking transformation \blsd.
On the one hand, the unrestricted domain blocking transformation
induces a larger search space, as the bodies of the clauses $x\approx y
\lor x \not\approx y \gets \dom(x) \land \dom(y)$ are less constrained
than their counterparts in the subterm domain blocking transformation.
This becomes obvious after extending the
clause body of $x\approx y \lor x \not\approx y \gets \subterm(x,
y)$ from the $\blsd$ transformation with $\dom(x) \land \dom(y)$, which does not change anything.
On the other hand, the unrestricted domain blocking transformation
enables the finding of models with smaller domains.
This means fewer congruence classes on the Herbrand terms are induced
by the equality relation~$\approx$.
As our experiments show, such models can often be found quicker
in satisfiable problems, even for the $\crr$ transformation.

Using the ideas of the termination proof in~\cite{SchmidtTishkovsky13a}
for semantic ground tableau with unrestricted domain blocking for
description logics with the expressive power similar to the two-variable
fragment of first-order logic,
it can be shown BUMG with unrestricted domain blocking can return
finite models, if they exist,
even for problems of undecidable fragments.
Carrying over also the results in~\cite{SchmidtTishkovsky08b} implies
unrestricted domain blocking can be used in BUMG methods to return domain
minimal models for logics with the effective finite model property.

\subsection{Unrestricted Predicate Blocking}

The definition of the last variant of blocking, the \emph{unrestricted
(unary)  predicate blocking transformation}, is as follows.
  \begin{description}
  \item[(0) Initialization.] Initially, let $\blup(M) := M$.
\item[(1) Term equality case analysis.] 
 Extend $\blup(M)$ by these clauses, for each unary predicate symbol
 $p \in \Sigma_P$.
 \begin{equation*}
   x\approx y \lor x \not\approx y  \gets p(x) \land p(y)
 \end{equation*}
Finally, add the clause 
\[
{}\gets x\approx y \land x \not\approx y
\]
to $\blup(M)$ .
\end{description}
This transformation allows to equate any two (distinct) terms in a
$p$-relation, if there are any. The motivation is a
combination of the above, to block cycles on $p$-literals if they
arise, and to compute models with small domains.

\subsection{Comparison on an Example}
\label{section_DL_example}

It is instructive to compare the effects on the returned models of
the four blocking transformations on an example from
description logics.  To this end, consider the description
logic knowledge base (left) and its translation into
clause logic (right) in Table~\ref{table_DL_example}.
Notice that the   cycle in the inclusion statements in
the TBox (for~$p_1$ and~$p_2$) means some form of blocking is needed
for decidability in tableau-based description logic systems.
Likewise, blocking is needed  to
force BUMG methods to terminate on the translated clause form. 
Any of the four blocking transformations defined above suffice.
Table~\ref{table_blocking_comparison_by_example} summarizes the
behaviour of these transformations, in terms of interesting relations
in the computed model.

\begin{table}[tbh]
\caption{Sample description logic knowledge base and clausal form}
\label{table_DL_example}
\begin{tabular}{cc}
  TBox & ABox \\\hline \hline
  \begin{math}
    \begin{aligned}
\rule{0pt}{2.7ex}  \ID p_1 & \sqsubseteq \exists\: \ID r . \ID p_2 \\
      \ID p_2 & \sqsubseteq \exists\: \ID r . \ID p_1 \\
\rule[-1.5ex]{0pt}{-1.5ex}      \ID p_1 & \sqsubseteq \exists\: \ID s . \ID q 
    \end{aligned}
  \end{math}
& 
\begin{math}
  \begin{aligned}
    \ID p_1(\ID a) & \\
    \ID p_1(\ID b) &
  \end{aligned}
\end{math}
\end{tabular}
\qquad\qquad\qquad
\begin{math}
\begin{aligned}
  \ID p_2(\ID f(x)) & \leftarrow \ID p_1(x) & 
  \ID q(\ID h(x)) & \leftarrow \ID p_1(x) \\  
\ID r(x,\ID f(x)) & \leftarrow \ID p_1(x)  & \quad 
  \ID s(x,\ID h(x)) & \leftarrow \ID p_1(x)\\
\ID p_1(\ID g(x)) & \leftarrow \ID p_2(x) &
  \ID p_1(\ID a) & \leftarrow   \\
\ID r(x,\ID g(x)) & \leftarrow \ID p_2(x) &
  \ID p_1(\ID b) & \leftarrow
\end{aligned}
\end{math}
\end{table}

\begin{table}[hbt]
\caption{Partial truth assignments in models computed for the sample knowledge base}
\label{table_blocking_comparison_by_example}
$
\renewcommand{\arraystretch}{1.2}
\begin{array}{l@{\qquad}l@{\qquad}l@{\qquad}l@{\qquad}l@{\qquad}l}
\multicolumn{1}{l}{\text{Blocking}} & 
\multicolumn{1}{l}{\dom} &
\multicolumn{1}{l}{\approx} &
\multicolumn{1}{l}{\ID p_1} & 
\multicolumn{1}{l}{\ID p_2} & 
\multicolumn{1}{l}{\ID q} \\\hline\hline
\blsd & 
\ID a, \ID b &
\begin{array}[t]{@{}l}
\IDall{f(a)\approx a}, \IDall{f(b)\approx b}, \\
\IDall{g(a)\approx a}, \IDall{g(b)\approx b},\\
\IDall{h(a)\approx a}, \IDall{h(b)\approx b}
\end{array} &
\ID a, \ID b &
\ID a, \ID b &
\ID a, \ID b \\\hline
\blsp & 
\begin{array}[t]{@{}l}
\ID a, \ID b, \\
\IDall{f(a)}, \IDall{f(b)}, \\
\IDall{h(a)}, \IDall{h(b)}
\end{array} &
\begin{array}[t]{l}
\IDall{g(f(a))\approx a},\\
\IDall{g(f(b))\approx b}
\end{array} &
\ID a, \ID b &
\IDall{f(a)}, \IDall{f(b)} &
\IDall{h(a)}, \IDall{h(b)} \\\hline
\blud & 
\ID b &
\begin{array}[t]{@{}l}
\IDall{a\approx b}, \IDall{f(b)\approx b},\\
\IDall{g(b)\approx b}, \IDall{h(b)\approx b}
\end{array} &
\ID b &
\ID b &
\ID b\\\hline
\blup & 
\ID b, \IDall{f(b)}, \IDall{h(b)} &
\begin{array}[t]{@{}l}
\IDall{a\approx b},\\
\IDall{g(f(b))\approx b} 
\end{array} &
\ID b &
\IDall{f(b)} &
\IDall{h(b)}
\end{array}
$
\end{table}

When comparing in detail the blocking techniques developed for description
logics it becomes clear
that
\[
\text{the transformations $\rr \circ \tau$ and $\sh \circ\rr \circ \tau$, for $\tau \in \{ \blsd, \blsp,
\blud, \blup \}$},
\]
when applied to a knowledge base
with the finite model property, 
in conjunction with a suitable BUMG method (see above), can be refined to
simulate various forms of standard blocking techniques used in
description logic systems,
including subset ancestor blocking and equality ancestor
blocking, cf.~\cite{HustadtSchmidt99b}, \cite{SchmidtTishkovsky13a}
and~\cite{KhodadadiSchmidtTishkovsky13a}.
Because standard loop checking mechanisms used in description logic
systems do not require backtracking, 
appropriate search strategies and restrictions for performing inferences
and applying blocking need to be used.

An advantage of our approach to blocking as opposed to blocking
without equality reasoning used in mainstream description logic
systems~\citep{BaaderSattler01} is that it applies to any first-order
clause set, not only to clauses from the translation of description
logic problems.
This makes the approach very general and widely applicable.

For instance, our approach makes it possible to extend description
logics with arbitrary (first-order expressible) `rule' languages.
`Rules' provide a connection to (deductive) databases and are being
used to represent information that is currently not expressible in
the description logics associated with OWL~DL.
The specification of many natural properties of binary relations and
complex statements involving binary relations are
outside the scope of most current description logic systems.
An example is the statement:
individuals who live and work at the same location are home
workers.
This can be expressed as a Horn rule (clause)
\[
\ID{homeWorker}(x) \gets \ID{work}(x, y)\land \ID{live}(x, z )
\land \ID{loc}(y,w) \land \ID{loc}(z,w),
\]
but, with some
exceptions~\citep{HustadtSchmidtWeidenbach99,WeidenbachSchmidtEtAl07},
is not expressible in current description logic systems.

\section{Soundness and Completeness of the Transformations}
\label{sec:soundness:completeness}

Each of the blocking transformations is complete:

\begin{proposition}[Completeness of blocking wrt.\
  E-interpretations]
\label{result_completeness_blocking}
Let $M$ be any clause set.  
For all $\tau \in \{ \blsd,\ \blsp,\ \blud,\ \blup \}$,  
if $\tau(M)$ is E-satisfiable then~$M$ is E-satisfiable.
\end{proposition}
\begin{proof}
  Not difficult, as $M \subseteq \tau(M)$ by definition. \qed
\end{proof}
The converse, that is, soundness of the transformation, is easy to prove.
One basically needs to observe that the clauses added in respectively
Steps~(2) and (1) of the blocking transformations, realize a case distinction over whether two terms
are equal or not.
Trivially, one of the two cases always holds.

Putting all the transformations and the corresponding results together we can 
state the main theoretical result of the paper.

\begin{theorem}[Completeness of the combined transformations with
respect to E-interpretations]
\label{result_combined_translations_completeness}
  Let $M$ be a clause set and suppose $\tr$ is any of the
  transformations in $\{ \rr, \sh \circ \rr \} \cup \{ \rr \circ \tau,
  \sh \circ \rr \circ \tau \mid \tau \in \{ \blsd,\ \blsp,\ \blud,\ \blup \} \}$ or 
  $\{ \crr, \sh \circ \crr \} \cup \{ \crr \circ \tau,
  \sh \circ \crr \circ \tau \mid \tau \in \{ \blsd,\ \blsp,\ \blud,\ \blup \} \}$. 
Then: 
\begin{enumerate}[(i)]
\item
$\tr(M)$ is range-restricted. 
\item
$\tr(M)$ can be computed in quadratic time.
\item
If $\tr(M) \cup \{ x\approx x \gets \dom(x) \}$ is E-satisfiable then $M$ is E-satisfiable.
\end{enumerate}
\end{theorem}
The reverse directions of (iii), that is, soundness of the respective
transformations, hold as well.  
The proofs are either easy or completely standard. %, and therefore omitted here.

By carefully modifying the definition of $\rr$ 
it is possible to compute the reductions in linear
time.
\begin{proposition}
Let $M$ and $\tr$ be as in the previous result. Then:
\begin{enumerate}[(i)]
\item
The size of $\tr(M)$ is bounded by a linear function
in the size of $M$.
\item
$\tr(M)$ can be computed in linear time.
\end{enumerate}
\end{proposition}

\section{Decidability of \BS classes}
\label{sec:other}

The Bernays-Sch\"onfinkel class can be decided using transformations
into range-restricted clauses.
Formulae in the Bernays-Sch\"onfinkel class are conjunctions of
function-free and equality-free formulae of the form
$\exists^*\forall^* \psi$, where $\psi$ is free of quantifiers.
A clause is a \emph{BS clause} iff all functional terms occurring in it are
constants.

It is proved in~\cite{SchmidtHustadt05b} that hyperresolution and any
refinements decide the class of range-restricted \BS clauses without equality.
Here assume that the language includes equality.
\begin{theorem}
The class of range-restricted \BS clauses (with equality),
is decidable by hyperresolution (and paramodulation) and all refinements.
\end{theorem}
This means all refinements of hyperresolution (and some form of equality
reasoning) combined with any translation into range-restricted clauses
is a decision procedure for the \BS class.

Therefore:
\begin{corollary}
Let $M$ be any set of \BS clauses, and suppose $\tr$ is any of the
transformations 
in 
$\{ \rr, \sh \circ \rr \} \cup \{ \rr \circ \tau,
  \sh \circ \rr \circ \tau \mid \tau \in \{ \blsd,\ \blsp,\ \blud,\ \blup \} \}$
and
$\{ \crr, \sh \circ \crr \} \cup \{ \crr \circ \tau,
  \sh \circ \rr \circ \tau \mid \tau \in \{ \blsd,\ \blsp,\ \blud,\ \blup \} \}$. 
Then:
\begin{enumerate}[(i)]
\item
Hyperresolution and all refinements decide $\tr(M)$.
\item
All BUMG methods decide~$M$.
\end{enumerate}
\end{corollary}

Since there are linear transformations of first-order formulae
into clausal form, and since all the $\tr$ transformations are effective
reductions of first-order clauses into range-restricted clauses,
we obtain the following result.

\begin{theorem}
\label{result_robust_decidability_BS}
\
\begin{enumerate}[(i)]
\item
There is a quadratic (linear), satisfiability equivalence preserving
transformation of any formula in the Bernays-Sch\"onfinkel class,
and any set of \BS clauses, into
a set of range-restricted \BS clauses.
\item
All procedures based on hyperresolution or BUMG decide the class
of \BS formulae and the class of \BS clauses.
\end{enumerate}
\end{theorem}

In \cite{SchmidtHustadt05b} a similar but different transformation is
used to prove this result for hyperresolution and \BS without equality.
In fact, what is crucial for deciding the \BS class is a grounding method.
This can be achieved by any form of range-restriction and
hyperresolution-like inferences. 
Theorem~\ref{result_robust_decidability_BS}.(ii) can therefore be 
strengthened to include also any instantiation-based method, in
particular also methods using on-the-fly instantiation such as semantic
Smullyan-type tableaux.

\section{Experimental Evaluation}
\label{sec:experiments}

We have implemented the transformations described in the previous
sections and carried out experiments on problems from the TPTP library,
Version~6.0.0.
The implementation, in SWI-Prolog, is called Yarralumla (Yet another
range-restriction avoiding loops under much less assumptions).
Since the transformations introduced in this paper are defined for
clausal problems we have selected for the experiments all the CNF problems
from the TPTP suite.

In our initial research~\citep{BaumgartnerSchmidt06} we used Yarralumla
with the MSPASS theorem prover,
Version~2.0g.1.4~\citep{HustadtSchmidt00b}.
As the extra features of MSPASS have in the mean time been integrated into the
SPASS theorem prover~\citep{WeidenbachSchmidtEtAl07} and SPASS has
significantly evolved since Version~2.0, for the present paper we
combined Yarralumla with SPASS Version~3.8d as a BUMG system.

For that purpose we modified the code of SPASS in a number of ways.
We added one new flag to activate splitting on
positive ground equality literals in positive non-Horn clauses.
The main inference loop was updapted so that finding a splitting clause
and applying splitting has highest priority (unchanged)
followed immediately by picking a non-positive blocking clause,
that is, clauses of the form $s\approx t \lor H_1 \lor \cdots \lor
H_m \gets B_1 \land \cdots \land B_k$ for $m \geq 0$ and $k > 0$,
and performing inferences with it.
The selection of splitting clauses was adapted so that positive
blocking clauses are always selected, when there are any.
Moreover, the first equality literal is split upon.
Positive blocking clauses are ground clauses of the form $s\approx
t \lor H_1 \lor \cdots \lor H_m$, where $m \geq 1$.
This adaptation ensures blocking is performed eagerly to
keep the set of ground terms small.
The tests with Yarralumla were performed using ordered resolution and
superposition with selection of at least one negative literal, forward and
backward rewriting, unlimited splitting and matching replacement
resolution, subsumption deletion and various other simplification rules.
This means the inferences are performed in an ordered
hyperresolution-style with eager splitting and forward and backward
ground rewriting.
The derivations constructed are thus BUMG tree derivations, the proofs
produced are BUMG refutation proofs, and the models returned
are BUMG models.

We also tested SPASS Version~3.8d in auto mode on the sample.
In auto mode SPASS used ordered resolution with dynamic
selection.
SPASS automatically turned off splitting for non-Horn clauses.
Dynamic selection means typically literals were only selected if
multiple maximal literals occur in a clause.
This means the behaviour of SPASS in auto mode was very different to
that of SPASS-Yarralumla, which always selected a literal in clauses
with non-empty negative part.
The changes to SPASS in SPASS-Yarralumla meant that splitting was
performed eagerly and blocking clauses were targeted, which was not the
case with SPASS in auto mode. 
We tested SPASS in auto mode only on the original files (translated from
TPTP syntax to SPASS syntax).

The experiments were run on a cluster of 128 Dell PowerEdge M610
Blade Servers each with two Intel Xeon E5620 2.4~GHz processors and 48 GiB main
memory each.
The time limit was ten minutes (CPU time).

SPASS-Yarralumla can be downloaded from\\ \url{http://www.cs.man.ac.uk/~schmidt/spass-yarralumla/}.

\subsection{Results}

\begin{sidewaystable}[!htbp]
\caption{Number of problems solved on satisfiable problems, by TPTP categories.}
\label{table_count_sat}

\smallskip
\begin{tabular}{l@{\ }r@{\quad\ }r@{\ }r@{\ }r@{\ }r@{\ }r@{\quad\ }r@{\ }r@{\ }r@{\ }r@{\ }r@{\quad\ }r@{\ }r@{\ }r@{\ }r@{\ }r@{\quad\ }r@{\ }r@{\ }r@{\ }r@{\ }r@{\quad\ }r}
Category\rule{0 pt}{40pt}
&    
\# & 
\begin{rotate}{45} $\rr$ \end{rotate} & % SPASS_yarralumla_splitting
\begin{rotate}{45} $\rr\circ\blsd$ \end{rotate} &  % SPASS_yarralumla_splitting_bld
\begin{rotate}{45} $\rr\circ\blud$ \end{rotate} & % SPASS_yarralumla_splitting_bld_ubl
\begin{rotate}{45} $\rr\circ\blsp$ \end{rotate} & %  SPASS_yarralumla_splitting_blu
\begin{rotate}{45} $\rr\circ\blup$ \end{rotate} & % SPASS_yarralumla_splitting_blu_ubl
\begin{rotate}{45} $\sh\circ\rr$ \end{rotate} & %  SPASS_yarralumla_splitting_sh
\begin{rotate}{45} $\sh\circ\rr\circ\blsd$ \end{rotate} & % SPASS_yarralumla_splitting_sh_bld
\begin{rotate}{45} $\sh\circ\rr\circ\blud$ \end{rotate} & % SPASS_yarralumla_splitting_sh_bld_ubl
\begin{rotate}{45} $\sh\circ\rr\circ\blsp$ \end{rotate} & % SPASS_yarralumla_splitting_sh_blu
\begin{rotate}{45} $\sh\circ\rr\circ\blup$ \end{rotate} & % SPASS_yarralumla_splitting_sh_blu_ubl
\begin{rotate}{45} $\crr$ \end{rotate} & % SPASS_yarralumla_splitting_stdrr
\begin{rotate}{45} $\crr\circ\blsd$ \end{rotate} & %   SPASS_yarralumla_splitting_stdrr_bld
\begin{rotate}{45} $\crr\circ\blud$ \end{rotate} & %   SPASS_yarralumla_splitting_stdrr_bld_ubl
\begin{rotate}{45} $\crr\circ\blsp$ \end{rotate} & %   SPASS_yarralumla_splitting_stdrr_blu
\begin{rotate}{45} $\crr\circ\blup$ \end{rotate} & %   SPASS_yarralumla_splitting_stdrr_blu_ubl
\begin{rotate}{45} $\sh\circ\crr$ \end{rotate} & %   SPASS_yarralumla_splitting_stdrr_sh
\begin{rotate}{45} $\sh\circ\crr\circ\blsd$ \end{rotate} & %   SPASS_yarralumla_splitting_stdrr_sh_bld
\begin{rotate}{45} $\sh\circ\crr\circ\blud$ \end{rotate} & %   SPASS_yarralumla_splitting_stdrr_sh_bld_ubl
\begin{rotate}{45} $\sh\circ\crr\circ\blsp$ \end{rotate} & %   SPASS_yarralumla_splitting_stdrr_sh_blu
\begin{rotate}{45} $\sh\circ\crr\circ\blup$ \end{rotate} & %   SPASS_yarralumla_splitting_stdrr_sh_blu_ubl
\begin{rotate}{45} auto \end{rotate} %   SPASS38d
\\
\end{tabular}
\end{sidewaystable}

Tables~\ref{table_count_sat} and~\ref{table_count_ratings_sat_summary}
summarize the results for satisfiable clausal
problems in the TPTP library, measuring the number of problems solved
with in the time limit.
The columns with the heading `\#' give the number of
problems in the TPTP categories and the different TPTP rating ranges. 
The subsequent columns give the number of problems solved within the
time limit.
The results are presented for the different BUMG methods that were used.
For example, $\sh\circ\rr\circ\blsd$ refers to the method based on the
transformation defined by the new range-restriction transformation,
shifting and subterm domain blocking.
To evaluate the effect of the different forms of blocking the results
are grouped into groups of five: no blocking, subterm domain blocking
($\blsd$), unrestricted domain blocking ($\blud$), subterm predicate
blocking ($\blsp$) and unrestricted predicate blocking ($\blsp$).
In each group the first column provides the \emph{baseline} for that group.
The last column with the heading `auto' gives the results for runs
of SPASS Version~3.8d in auto mode on the original input files.
The runtimes for the problems solved spanned the whole range, from less
than one second to all of the time allowed.

\begin{sidewaystable}[!htbp]
\caption{Number of problems solved on satisfiable problems, by TPTP problem rating.}
\label{table_count_ratings_sat_summary}

\smallskip
\begin{tabular}{l@{\ }r@{\quad\ }r@{\ }r@{\ }r@{\ }r@{\ }r@{\quad\ }r@{\ }r@{\ }r@{\ }r@{\ }r@{\quad\ }r@{\ }r@{\ }r@{\ }r@{\ }r@{\quad\ }r@{\ }r@{\ }r@{\ }r@{\ }r@{\quad\ }r}
Category\rule{0 pt}{40pt}
&    
\# & 
\begin{rotate}{45} $\rr$ \end{rotate} & % SPASS_yarralumla_splitting
\begin{rotate}{45} $\rr\circ\blsd$ \end{rotate} &  % SPASS_yarralumla_splitting_bld
\begin{rotate}{45} $\rr\circ\blud$ \end{rotate} & % SPASS_yarralumla_splitting_bld_ubl
\begin{rotate}{45} $\rr\circ\blsp$ \end{rotate} & %  SPASS_yarralumla_splitting_blu
\begin{rotate}{45} $\rr\circ\blup$ \end{rotate} & % SPASS_yarralumla_splitting_blu_ubl
\begin{rotate}{45} $\sh\circ\rr$ \end{rotate} & %  SPASS_yarralumla_splitting_sh
\begin{rotate}{45} $\sh\circ\rr\circ\blsd$ \end{rotate} & % SPASS_yarralumla_splitting_sh_bld
\begin{rotate}{45} $\sh\circ\rr\circ\blud$ \end{rotate} & % SPASS_yarralumla_splitting_sh_bld_ubl
\begin{rotate}{45} $\sh\circ\rr\circ\blsp$ \end{rotate} & % SPASS_yarralumla_splitting_sh_blu
\begin{rotate}{45} $\sh\circ\rr\circ\blup$ \end{rotate} & % SPASS_yarralumla_splitting_sh_blu_ubl
\begin{rotate}{45} $\crr$ \end{rotate} & % SPASS_yarralumla_splitting_stdrr
\begin{rotate}{45} $\crr\circ\blsd$ \end{rotate} & %   SPASS_yarralumla_splitting_stdrr_bld
\begin{rotate}{45} $\crr\circ\blud$ \end{rotate} & %   SPASS_yarralumla_splitting_stdrr_bld_ubl
\begin{rotate}{45} $\crr\circ\blsp$ \end{rotate} & %   SPASS_yarralumla_splitting_stdrr_blu
\begin{rotate}{45} $\crr\circ\blup$ \end{rotate} & %   SPASS_yarralumla_splitting_stdrr_blu_ubl
\begin{rotate}{45} $\sh\circ\crr$ \end{rotate} & %   SPASS_yarralumla_splitting_stdrr_sh
\begin{rotate}{45} $\sh\circ\crr\circ\blsd$ \end{rotate} & %   SPASS_yarralumla_splitting_stdrr_sh_bld
\begin{rotate}{45} $\sh\circ\crr\circ\blud$ \end{rotate} & %   SPASS_yarralumla_splitting_stdrr_sh_bld_ubl
\begin{rotate}{45} $\sh\circ\crr\circ\blsp$ \end{rotate} & %   SPASS_yarralumla_splitting_stdrr_sh_blu
\begin{rotate}{45} $\sh\circ\crr\circ\blup$ \end{rotate} & %   SPASS_yarralumla_splitting_stdrr_sh_blu_ubl
\begin{rotate}{45} auto \end{rotate} %   SPASS38d
\\
\end{tabular}

Note: ${}^\dagger$1017 = 150 SAT + 28 OPN + 839 UNK
\end{sidewaystable}

The best results in each group in each row are highlighted in bold font.
The underlined values are the best results for all 
methods including SPASS in auto mode.
As expected the worst results in each group were obtained for the
baseline transformations without blocking.
This confirms the expectation that blocking is an essential technique
for BUMG methods.
Among the different blocking techniques the best
results were obtained with unrestricted domain blocking in all four groups.
Overall, the best result was obtained for the combination with
$\rr$ and shifting, i.e., $\sh\circ\rr\circ\blud$, solving 6.0\%~more
problems than the second best method, $\crr\circ\blud$ using the classical range-restriction
transformation without shifting, and nearly 11\%~more problems than
the transformations $\rr\circ\blud$ and $\sh\circ\crr\circ\blud$.
This means shifting had a significant positive effect in combination
with the new range-restriction transformation, but less so in
combination with classical range-restriction.
The positive effect of shifting could also be seen for the number of
problems solved without blocking for $\rr$ and $\sh\circ\rr$ (34\%~improvement).

The good results for $\crr\circ\blud$ show the value of classical range-restriction.
In the LAT category, $\crr\circ\blud$ solved 32~problems, whereas
$\sh\circ\rr\circ\blud$ solved only 5~problems.
This seems to indicate there was a trade-off between using the 
$\crr$ transformation and the $\rr$ transformation in combination with shifting, but also
showed the virtues of unrestricted domain blocking as a universal
technique for BUMG.
SPASS in auto mode fared very well in the SWV category, where
79~problems were solved compared to 7--8~problems for the best BUMG methods.
Overall SPASS in auto mode solved 9\%~fewer problems than the best
BUMG method $\sh\circ\rr\circ\blud$.

Looking at the top half of Table~\ref{table_count_ratings_sat_summary}
(up to difficulty rating of $0.40$), the BUMG method
based on $\sh\circ\rr\circ\blud$ fared best, but for problems more
difficult (up to a rating of $0.70$) the performance deteriorated and
the method $\crr\circ\blud$ solved the highest number of problems.
For problems with ratings higher than $0.70$ SPASS in auto mode solved
significantly more problems than the BUMG methods.
One problem with rating $1.00$ was solved by the $\crr\circ\blud$ method (namely,
GRP741-1 in 121.86~seconds).
Problems in the TPTP library with rating~1.00 have not yet
been solved by any other prover.

\begin{table}[!htbp]
\caption{Evaluation of blocking techniques.}
\label{table_uniquely_solved_sat_unsat_wrt_baseline}
\begin{tabular}{l@{\quad}l@{\quad}r@{\quad}r@{\quad}r@{\quad}r@{\quad}r@{\quad\quad}r@{\quad}r@{\quad}r@{\quad}r@{\quad}r@{\quad}r@{\quad\quad}r@{\quad}r@{\quad}r@{\quad}r@{\quad}r@{\quad}r@{\quad\quad}r@{\quad}r@{\quad}r@{\quad}r@{\quad}r@{\ }r}
\textbf{Satisfiable}
&    
Baseline & 
$\blsd$ & % bld
$\blud$ & % bld_ubl
$\blsp$ & % blu
$\blup$ & % blu_ubl
\\ \hline \hline
& $\rr$  %  SPASS_yarralumla_splitting
&    -0 /
   +212 % SPASS_yarralumla_splitting_bld | Completion found |
&\bf -4 / % SPASS_yarralumla_splitting         | Completion found |
   +238 % SPASS_yarralumla_splitting_bld_ubl | Completion found |
&    -0 /
   +169 % SPASS_yarralumla_splitting_blu | Completion found |
&    -1 / % SPASS_yarralumla_splitting         | Completion found |
    +81 % SPASS_yarralumla_splitting_blu_ubl | Completion found |
\\
            & $\sh\circ\rr$ % SPASS_yarralumla_splitting_sh
&    -1 / % SPASS_yarralumla_splitting_sh     | Completion found |
   +195 % SPASS_yarralumla_splitting_sh_bld | Completion found |
&\bf -5 / % SPASS_yarralumla_splitting_sh         | Completion found |
   +225 % SPASS_yarralumla_splitting_sh_bld_ubl | Completion found |
&    -2 / % SPASS_yarralumla_splitting_sh     | Completion found |
   +168 % SPASS_yarralumla_splitting_sh_blu | Completion found |
&    -2 / % SPASS_yarralumla_splitting_sh         | Completion found |
    +77 % SPASS_yarralumla_splitting_sh_blu_ubl | Completion found |
\\
            & $\crr$ % SPASS_yarralumla_splitting_stdrr
&    -0 /
   +249 % SPASS_yarralumla_splitting_stdrr_bld | Completion found |
&\bf -4 / % SPASS_yarralumla_splitting_stdrr         | Completion found |
   +294 % SPASS_yarralumla_splitting_stdrr_bld_ubl | Completion found |
&    -0 /
   +184 % SPASS_yarralumla_splitting_stdrr_blu | Completion found |
&    -3 / % SPASS_yarralumla_splitting_stdrr         | Completion found |
    +78 % SPASS_yarralumla_splitting_stdrr_blu_ubl | Completion found |
\\
            & $\sh\circ\crr$ % SPASS_yarralumla_splitting_stdrr_sh
&    -0 /
   +226 % SPASS_yarralumla_splitting_stdrr_sh_bld | Completion found |
&\bf -5 / % SPASS_yarralumla_splitting_stdrr_sh         | Completion found |
   +274 % SPASS_yarralumla_splitting_stdrr_sh_bld_ubl | Completion found |
&    -0 /
   +168 % SPASS_yarralumla_splitting_stdrr_sh_blu | Completion found |
&    -3 / % SPASS_yarralumla_splitting_stdrr_sh         | Completion found |
    +77 % SPASS_yarralumla_splitting_stdrr_sh_blu_ubl | Completion found |
\\[2\abovedisplayskip]
\textbf{Unsatisfiable} &
\\ \hline \hline
& $\rr$  %  SPASS_yarralumla_splitting
&  -211 / % SPASS_yarralumla_splitting     | Proof found      |
    +78 % SPASS_yarralumla_splitting_bld | Proof found      |
&  -315 / % SPASS_yarralumla_splitting         | Proof found      |
    \bf +83 % SPASS_yarralumla_splitting_bld_ubl | Proof found      |
&  -100 / % SPASS_yarralumla_splitting     | Proof found      |
    +61 % SPASS_yarralumla_splitting_blu | Proof found      |
&   {\bf -77} / % SPASS_yarralumla_splitting         | Proof found      |
    +37 % SPASS_yarralumla_splitting_blu_ubl | Proof found      |
\\
              & $\sh\circ\rr$ % SPASS_yarralumla_splitting_sh
&  -188 / % SPASS_yarralumla_splitting_sh     | Proof found      |
   +106 % SPASS_yarralumla_splitting_sh_bld | Proof found      |
&  -225 / % SPASS_yarralumla_splitting_sh         | Proof found      |
   \bf +126 % SPASS_yarralumla_splitting_sh_bld_ubl | Proof found      |
&   -81 / % SPASS_yarralumla_splitting_sh     | Proof found      |
    +87 % SPASS_yarralumla_splitting_sh_blu | Proof found      |
&   {\bf -34} / % SPASS_yarralumla_splitting_sh         | Proof found      |
    +52 % SPASS_yarralumla_splitting_sh_blu_ubl | Proof found      |
\\
              & $\crr$ % SPASS_yarralumla_splitting_stdrr
&   -65 / % SPASS_yarralumla_splitting_stdrr     | Proof found      |
   +170 % SPASS_yarralumla_splitting_stdrr_bld | Proof found      |
&  -105 / % SPASS_yarralumla_splitting_stdrr         | Proof found      |
   \bf +242 % SPASS_yarralumla_splitting_stdrr_bld_ubl | Proof found      |
&   -32 / % SPASS_yarralumla_splitting_stdrr     | Proof found      |
    +78 % SPASS_yarralumla_splitting_stdrr_blu | Proof found      |
&   {\bf -26} / % SPASS_yarralumla_splitting_stdrr         | Proof found      |
    +24 % SPASS_yarralumla_splitting_stdrr_blu_ubl | Proof found      |
\\
              & $\sh\circ\crr$ % SPASS_yarralumla_splitting_stdrr_sh
&   -52 / % SPASS_yarralumla_splitting_stdrr_sh     | Proof found      |
   +163 % SPASS_yarralumla_splitting_stdrr_sh_bld | Proof found      |
&   -98 / % SPASS_yarralumla_splitting_stdrr_sh         | Proof found      |
   \bf +190 % SPASS_yarralumla_splitting_stdrr_sh_bld_ubl | Proof found      |
&   -30 / % SPASS_yarralumla_splitting_stdrr_sh     | Proof found      |
    +57 % SPASS_yarralumla_splitting_stdrr_sh_blu | Proof found      |
&   {\bf -16} / % SPASS_yarralumla_splitting_stdrr_sh         | Proof found      |
    +15 % SPASS_yarralumla_splitting_stdrr_sh_blu_ubl | Proof found      |
\end{tabular}
\end{table}

Table~\ref{table_uniquely_solved_sat_unsat_wrt_baseline} presents an
evaluation of the different blocking techniques,
listing the number of problems lost and the number of problems gained
against the baseline methods in each group.
The results confirm the significant positive effect of unrestricted
domain blocking for satisfiable problems.

Analysis of the gain and loss of the method based on
$\sh\circ\rr\circ\blud$ against the other methods gave these results:
Against $\rr\circ\blud$ 66~problems were gained and 20~problems lost;
against $\sh\circ\crr\circ\blud$ the gain/loss was \text{+90/-45} and against
$\crr\circ\blud$ it was +85/-59.
This non-uniformity suggests each variation of range-restriction had
the potential to solve some problems not solvable within the time
limit by $\sh\circ\rr$ with unrestricted blocking.
The biggest variation was against SPASS in auto mode, where
169 problems were gained and 130 problems were lost.

\begin{sidewaystable}[!htbp]
\caption{Uniquely solved problems.}
\label{table_uniquely_solved_sat}

\smallskip
\begin{tabular}{l@{\quad}r@{\quad}r@{\quad}r@{\quad}r@{\quad}r@{\quad}r@{\quad\quad}r@{\quad}r@{\quad}r@{\quad}r@{\quad}r@{\ }r@{\quad\quad}r@{\quad}r@{\quad}r@{\quad}r@{\quad}r@{\ }r@{\quad\quad}r@{\quad}r@{\quad}r@{\quad}r@{\quad}r@{\ }r@{\quad}r@{\ }r}
\textbf{Satisfiable}
\rule{0 pt}{40pt}
&    
\begin{rotate}{45} $\rr$ \end{rotate} & % SPASS_yarralumla_splitting
\begin{rotate}{45} $\rr\circ\blsd$ \end{rotate} &  % SPASS_yarralumla_splitting_bld
\begin{rotate}{45} $\rr\circ\blud$ \end{rotate} & % SPASS_yarralumla_splitting_bld_ubl
\begin{rotate}{45} $\rr\circ\blsp$ \end{rotate} & %  SPASS_yarralumla_splitting_blu
\begin{rotate}{45} $\rr\circ\blup$ \end{rotate} & % SPASS_yarralumla_splitting_blu_ubl
\begin{rotate}{45} Total\end{rotate} &
\begin{rotate}{45} $\sh\circ\rr$ \end{rotate} & %  SPASS_yarralumla_splitting_sh
\begin{rotate}{45} $\sh\circ\rr\circ\blsd$ \end{rotate} & % SPASS_yarralumla_splitting_sh_bld
\begin{rotate}{45} $\sh\circ\rr\circ\blud$ \end{rotate} & % SPASS_yarralumla_splitting_sh_bld_ubl
\begin{rotate}{45} $\sh\circ\rr\circ\blsp$ \end{rotate} & % SPASS_yarralumla_splitting_sh_blu
\begin{rotate}{45} $\sh\circ\rr\circ\blup$ \end{rotate} & % SPASS_yarralumla_splitting_sh_blu_ubl
\begin{rotate}{45} Total\end{rotate} &
\begin{rotate}{45} $\crr$ \end{rotate} & % SPASS_yarralumla_splitting_stdrr
\begin{rotate}{45} $\crr\circ\blsd$ \end{rotate} & %   SPASS_yarralumla_splitting_stdrr_bld
\begin{rotate}{45} $\crr\circ\blud$ \end{rotate} & %   SPASS_yarralumla_splitting_stdrr_bld_ubl
\begin{rotate}{45} $\crr\circ\blsp$ \end{rotate} & %   SPASS_yarralumla_splitting_stdrr_blu
\begin{rotate}{45} $\crr\circ\blup$ \end{rotate} & %   SPASS_yarralumla_splitting_stdrr_blu_ubl
\begin{rotate}{45} Total\end{rotate} &
\begin{rotate}{45} $\sh\circ\crr$ \end{rotate} & %   SPASS_yarralumla_splitting_stdrr_sh
\begin{rotate}{45} $\sh\circ\crr\circ\blsd$ \end{rotate} & %   SPASS_yarralumla_splitting_stdrr_sh_bld
\begin{rotate}{45} $\sh\circ\crr\circ\blud$ \end{rotate} & %   SPASS_yarralumla_splitting_stdrr_sh_bld_ubl
\begin{rotate}{45} $\sh\circ\crr\circ\blsp$ \end{rotate} & %   SPASS_yarralumla_splitting_stdrr_sh_blu
\begin{rotate}{45} $\sh\circ\crr\circ\blup$ \end{rotate} & %   SPASS_yarralumla_splitting_stdrr_sh_blu_ubl
\begin{rotate}{45} Total\end{rotate} &
\begin{rotate}{45} auto \end{rotate} & %   SPASS38d
\begin{rotate}{45} Total\end{rotate}
\\ \hline \hline
All methods
&     - %& XXXXXXXXXXXXXXXX & XXXXXXXXXXXXXXXXXXXXXXXXXXXXXXXXXXXXXX & 
&     1 % SPASS_yarralumla_splitting_bld              | Completion found |
&     - %& XXXXXXXXXXXXXXXX & XXXXXXXXXXXXXXXXXXXXXXXXXXXXXXXXXXXXXX & 
&     - %& XXXXXXXXXXXXXXXX & XXXXXXXXXXXXXXXXXXXXXXXXXXXXXXXXXXXXXX & 
&     - %& XXXXXXXXXXXXXXXX & XXXXXXXXXXXXXXXXXXXXXXXXXXXXXXXXXXXXXX & 
&
&     - %& XXXXXXXXXXXXXXXX & XXXXXXXXXXXXXXXXXXXXXXXXXXXXXXXXXXXXXX & 
&     - %& XXXXXXXXXXXXXXXX & XXXXXXXXXXXXXXXXXXXXXXXXXXXXXXXXXXXXXX & 
&     4 % SPASS_yarralumla_splitting_sh_bld_ubl       | Completion found |
&     - %& XXXXXXXXXXXXXXXX & XXXXXXXXXXXXXXXXXXXXXXXXXXXXXXXXXXXXXX & 
&     - %& XXXXXXXXXXXXXXXX & XXXXXXXXXXXXXXXXXXXXXXXXXXXXXXXXXXXXXX & 
&
&     - %& XXXXXXXXXXXXXXXX & XXXXXXXXXXXXXXXXXXXXXXXXXXXXXXXXXXXXXX & 
&     - %& XXXXXXXXXXXXXXXX & XXXXXXXXXXXXXXXXXXXXXXXXXXXXXXXXXXXXXX & 
&    11 % SPASS_yarralumla_splitting_stdrr_bld_ubl    | Completion found |
&     - %& XXXXXXXXXXXXXXXX & XXXXXXXXXXXXXXXXXXXXXXXXXXXXXXXXXXXXXX & 
&     - %& XXXXXXXXXXXXXXXX & XXXXXXXXXXXXXXXXXXXXXXXXXXXXXXXXXXXXXX & 
&
&     - %& XXXXXXXXXXXXXXXX & XXXXXXXXXXXXXXXXXXXXXXXXXXXXXXXXXXXXXX & 
&     4 % SPASS_yarralumla_splitting_stdrr_sh_bld     | Completion found |
&     1 % SPASS_yarralumla_splitting_stdrr_sh_bld_ubl | Completion found |
&     - %& XXXXXXXXXXXXXXXX & XXXXXXXXXXXXXXXXXXXXXXXXXXXXXXXXXXXXXX & 
&     - %& XXXXXXXXXXXXXXXX & XXXXXXXXXXXXXXXXXXXXXXXXXXXXXXXXXXXXXX & 
&
&   \bf\underline{115} % SPASS38d                                    | Completion found |
&   136 % NULL                                        | Completion found |
\\
All BUMG methods
&     - %& XXXXXXXXXXXXXXXX & XXXXXXXXXXXXXXXXXXXXXXXXXXXXXXXXXXXXXX & 
&     1 % SPASS_yarralumla_splitting_bld              | Completion found |
&     2 % SPASS_yarralumla_splitting_bld_ubl          | Completion found |
&     - %& XXXXXXXXXXXXXXXX & XXXXXXXXXXXXXXXXXXXXXXXXXXXXXXXXXXXXXX & 
&     - %& XXXXXXXXXXXXXXXX & XXXXXXXXXXXXXXXXXXXXXXXXXXXXXXXXXXXXXX & 
&
&     - %& XXXXXXXXXXXXXXXX & XXXXXXXXXXXXXXXXXXXXXXXXXXXXXXXXXXXXXX & 
&     - %& XXXXXXXXXXXXXXXX & XXXXXXXXXXXXXXXXXXXXXXXXXXXXXXXXXXXXXX & 
&     6 % SPASS_yarralumla_splitting_sh_bld_ubl       | Completion found |
&     - %& XXXXXXXXXXXXXXXX & XXXXXXXXXXXXXXXXXXXXXXXXXXXXXXXXXXXXXX & 
&     - %& XXXXXXXXXXXXXXXX & XXXXXXXXXXXXXXXXXXXXXXXXXXXXXXXXXXXXXX & 
&
&     - %& XXXXXXXXXXXXXXXX & XXXXXXXXXXXXXXXXXXXXXXXXXXXXXXXXXXXXXX & 
&     - %& XXXXXXXXXXXXXXXX & XXXXXXXXXXXXXXXXXXXXXXXXXXXXXXXXXXXXXX & 
&    \bf\underline{11} % SPASS_yarralumla_splitting_stdrr_bld_ubl    | Completion found |
&     - %& XXXXXXXXXXXXXXXX & XXXXXXXXXXXXXXXXXXXXXXXXXXXXXXXXXXXXXX & 
&     - %& XXXXXXXXXXXXXXXX & XXXXXXXXXXXXXXXXXXXXXXXXXXXXXXXXXXXXXX & 
&
&     - %& XXXXXXXXXXXXXXXX & XXXXXXXXXXXXXXXXXXXXXXXXXXXXXXXXXXXXXX & 
&     4 % SPASS_yarralumla_splitting_stdrr_sh_bld     | Completion found |
&     1 % SPASS_yarralumla_splitting_stdrr_sh_bld_ubl | Completion found |
&     - %& XXXXXXXXXXXXXXXX & XXXXXXXXXXXXXXXXXXXXXXXXXXXXXXXXXXXXXX & 
&     - %& XXXXXXXXXXXXXXXX & XXXXXXXXXXXXXXXXXXXXXXXXXXXXXXXXXXXXXX & 
&
&
&    25 % NULL                                        | Completion found |
\\
All BUMG, by group
&     - %& XXXXXXXXXXXXXXXX & XXXXXXXXXXXXXXXXXXXXXXXXXXXXXXXXXXXXXX & 
&    12 % SPASS_yarralumla_splitting_bld     | Completion found |
&    \bf 28 % SPASS_yarralumla_splitting_bld_ubl | Completion found |
&     - %& XXXXXXXXXXXXXXXX & XXXXXXXXXXXXXXXXXXXXXXXXXXXXXXXXXXXXXX & 
&     - %& XXXXXXXXXXXXXXXX & XXXXXXXXXXXXXXXXXXXXXXXXXXXXXXXXXXXXXX & 
&    40 % NULL                               | Completion found |
&     - %& XXXXXXXXXXXXXXXX & XXXXXXXXXXXXXXXXXXXXXXXXXXXXXXXXXXXXXX & 
&     4 % SPASS_yarralumla_splitting_sh_bld     | Completion found |
&    \bf 25 % SPASS_yarralumla_splitting_sh_bld_ubl | Completion found |
&     1 % SPASS_yarralumla_splitting_sh_blu     | Completion found |
&     2 % SPASS_yarralumla_splitting_sh_blu_ubl | Completion found |
&    32 % NULL                                  | Completion found |
&     - %& XXXXXXXXXXXXXXXX & XXXXXXXXXXXXXXXXXXXXXXXXXXXXXXXXXXXXXX & 
&    10 % SPASS_yarralumla_splitting_stdrr_bld     | Completion found |
&    \bf 57 % SPASS_yarralumla_splitting_stdrr_bld_ubl | Completion found |
&     - %& XXXXXXXXXXXXXXXX & XXXXXXXXXXXXXXXXXXXXXXXXXXXXXXXXXXXXXX & 
&     - %& XXXXXXXXXXXXXXXX & XXXXXXXXXXXXXXXXXXXXXXXXXXXXXXXXXXXXXX & 
&    67 % NULL                                     | Completion found |
&     - %& XXXXXXXXXXXXXXXX & XXXXXXXXXXXXXXXXXXXXXXXXXXXXXXXXXXXXXX & 
&    12 % SPASS_yarralumla_splitting_stdrr_sh_bld     | Completion found |
&    \bf 60 % SPASS_yarralumla_splitting_stdrr_sh_bld_ubl | Completion found |
&     - %& XXXXXXXXXXXXXXXX & XXXXXXXXXXXXXXXXXXXXXXXXXXXXXXXXXXXXXX & 
&     1 % SPASS_yarralumla_splitting_stdrr_sh_blu_ubl | Completion found |
&    73 % NULL                                        | Completion found |
\\
\\[2\abovedisplayskip]
\textbf{Unsatisfiable} &
\\ \hline \hline
All methods
&     - %& XXXXXXXXXXXXXXXX & XXXXXXXXXXXXXXXXXXXXXXXXXXXXXXXXXXXXXX & 
&     - %& XXXXXXXXXXXXXXXX & XXXXXXXXXXXXXXXXXXXXXXXXXXXXXXXXXXXXXX & 
&     1 % SPASS_yarralumla_splitting_bld_ubl | Proof found |
&     - %& XXXXXXXXXXXXXXXX & XXXXXXXXXXXXXXXXXXXXXXXXXXXXXXXXXXXXXX & 
&     - %& XXXXXXXXXXXXXXXX & XXXXXXXXXXXXXXXXXXXXXXXXXXXXXXXXXXXXXX & 
&
&     - %& XXXXXXXXXXXXXXXX & XXXXXXXXXXXXXXXXXXXXXXXXXXXXXXXXXXXXXX & 
&     - %& XXXXXXXXXXXXXXXX & XXXXXXXXXXXXXXXXXXXXXXXXXXXXXXXXXXXXXX & 
&     - %& XXXXXXXXXXXXXXXX & XXXXXXXXXXXXXXXXXXXXXXXXXXXXXXXXXXXXXX & 
&     - %& XXXXXXXXXXXXXXXX & XXXXXXXXXXXXXXXXXXXXXXXXXXXXXXXXXXXXXX & 
&     - %& XXXXXXXXXXXXXXXX & XXXXXXXXXXXXXXXXXXXXXXXXXXXXXXXXXXXXXX & 
&
&     - %& XXXXXXXXXXXXXXXX & XXXXXXXXXXXXXXXXXXXXXXXXXXXXXXXXXXXXXX & 
&     - %& XXXXXXXXXXXXXXXX & XXXXXXXXXXXXXXXXXXXXXXXXXXXXXXXXXXXXXX & 
&     - %& XXXXXXXXXXXXXXXX & XXXXXXXXXXXXXXXXXXXXXXXXXXXXXXXXXXXXXX & 
&     - %& XXXXXXXXXXXXXXXX & XXXXXXXXXXXXXXXXXXXXXXXXXXXXXXXXXXXXXX & 
&     - %& XXXXXXXXXXXXXXXX & XXXXXXXXXXXXXXXXXXXXXXXXXXXXXXXXXXXXXX & 
&
&     - %& XXXXXXXXXXXXXXXX & XXXXXXXXXXXXXXXXXXXXXXXXXXXXXXXXXXXXXX & 
&     - %& XXXXXXXXXXXXXXXX & XXXXXXXXXXXXXXXXXXXXXXXXXXXXXXXXXXXXXX & 
&     - %& XXXXXXXXXXXXXXXX & XXXXXXXXXXXXXXXXXXXXXXXXXXXXXXXXXXXXXX & 
&     - %& XXXXXXXXXXXXXXXX & XXXXXXXXXXXXXXXXXXXXXXXXXXXXXXXXXXXXXX & 
&     - %& XXXXXXXXXXXXXXXX & XXXXXXXXXXXXXXXXXXXXXXXXXXXXXXXXXXXXXX & 
&
&  \bf\underline{1779} % SPASS38d                           | Proof found |
&  1780 % NULL                               | Proof found |
\\
All BUMG methods
&    \bf\underline{17} % SPASS_yarralumla_splitting                  | Proof found |
&     4 % SPASS_yarralumla_splitting_bld              | Proof found |
&     5 % SPASS_yarralumla_splitting_bld_ubl          | Proof found |
&     2 % SPASS_yarralumla_splitting_blu              | Proof found |
&     2 % SPASS_yarralumla_splitting_blu_ubl          | Proof found |
&
&     3 % SPASS_yarralumla_splitting_sh               | Proof found |
&     5 % SPASS_yarralumla_splitting_sh_bld           | Proof found |
&     3 % SPASS_yarralumla_splitting_sh_bld_ubl       | Proof found |
&     - %& XXXXXXXXXXXXXXXX & XXXXXXXXXXXXXXXXXXXXXXXXXXXXXXXXXXXXXX & 
&     3 % SPASS_yarralumla_splitting_sh_blu_ubl       | Proof found |
&
&     - %& XXXXXXXXXXXXXXXX & XXXXXXXXXXXXXXXXXXXXXXXXXXXXXXXXXXXXXX & 
&     3 % SPASS_yarralumla_splitting_stdrr_bld        | Proof found |
&     3 % SPASS_yarralumla_splitting_stdrr_bld_ubl    | Proof found |
&     2 % SPASS_yarralumla_splitting_stdrr_blu        | Proof found |
&     - %& XXXXXXXXXXXXXXXX & XXXXXXXXXXXXXXXXXXXXXXXXXXXXXXXXXXXXXX & 
&
&     - %& XXXXXXXXXXXXXXXX & XXXXXXXXXXXXXXXXXXXXXXXXXXXXXXXXXXXXXX & 
&     1 % SPASS_yarralumla_splitting_stdrr_sh_bld     | Proof found |
&     3 % SPASS_yarralumla_splitting_stdrr_sh_bld_ubl | Proof found |
&     - %& XXXXXXXXXXXXXXXX & XXXXXXXXXXXXXXXXXXXXXXXXXXXXXXXXXXXXXX & 
&     - %& XXXXXXXXXXXXXXXX & XXXXXXXXXXXXXXXXXXXXXXXXXXXXXXXXXXXXXX & 
&
&
&    56 % NULL                                        | Proof found |
\\
All BUMG, by group
&    \bf 32 % SPASS_yarralumla_splitting         | Proof found |
&    15 % SPASS_yarralumla_splitting_bld     | Proof found |
&    18 % SPASS_yarralumla_splitting_bld_ubl | Proof found |
&    23 % SPASS_yarralumla_splitting_blu     | Proof found |
&     4 % SPASS_yarralumla_splitting_blu_ubl | Proof found |
&    92 % NULL                               | Proof found |
&    16 % SPASS_yarralumla_splitting_sh         | Proof found |
&    24 % SPASS_yarralumla_splitting_sh_bld     | Proof found |
&    \bf 41 % SPASS_yarralumla_splitting_sh_bld_ubl | Proof found |
&    29 % SPASS_yarralumla_splitting_sh_blu     | Proof found |
&    14 % SPASS_yarralumla_splitting_sh_blu_ubl | Proof found |
&   124 % NULL                                  | Proof found |
&     4 % SPASS_yarralumla_splitting_stdrr         | Proof found |
&    23 % SPASS_yarralumla_splitting_stdrr_bld     | Proof found |
&    \bf 90 % SPASS_yarralumla_splitting_stdrr_bld_ubl | Proof found |
&    10 % SPASS_yarralumla_splitting_stdrr_blu     | Proof found |
&     3 % SPASS_yarralumla_splitting_stdrr_blu_ubl | Proof found |
&   130 % NULL                                     | Proof found |
&     8 % SPASS_yarralumla_splitting_stdrr_sh         | Proof found |
&    31 % SPASS_yarralumla_splitting_stdrr_sh_bld     | Proof found |
&    \bf 71 % SPASS_yarralumla_splitting_stdrr_sh_bld_ubl | Proof found |
&     6 % SPASS_yarralumla_splitting_stdrr_sh_blu     | Proof found |
&     2 % SPASS_yarralumla_splitting_stdrr_sh_blu_ubl | Proof found |
&   118 % NULL                                        | Proof found |
\end{tabular}
\end{sidewaystable}

Table~\ref{table_uniquely_solved_sat} displays how many problems were
uniquely solved.
The first row lists how many problems were uniquely solved over all
methods including SPASS in auto mode.
Although two of the BUMG methods with unrestricted domain blocking
fared better than SPASS in auto mode, the latter solved a
significant number of problems that none of the BUMG could solve
(namely, 115~problems, or 27.5\% of the problems solved by SPASS in auto
mode, or 10.2\% of all satisfiable problems).
This reflected the orthogonality of the underlying methods.
Analogously, the relatively low number of problems uniquely solved
by the BUMG methods (21 problems, i.e., 1.9\% of all satisfiable
problems), which is also apparent from the number of problems
solved uniquely among the BUMG methods in the second row (25 or 2.2\%
of all satisfiable problems) can be attributed to the similarity of the
underlying methods.
An analysis of the number of uniquely solved problems per group of
BUMG methods in the third row of the table highlighted the importance
of unrestricted domain blocking.
While overall no problems were only solved with unary predicate
blocking techniques, within the groups there were four problems solved
only with unary predicate blocking.

\begin{table}[!htbp]
\caption{Average increase in the size of input files.}
\label{table_input_size_statistics}
\begin{tabular}[t]{lr}
Method & Avg % & Max
\\ \hline \hline
$\rr$ % SPASS_yarralumla_splitting
& 1.8 \\ % 13.7 \\
$\rr\circ\blsd$ % SPASS_yarralumla_splitting_bld
& 2.4 \\ % 13.7 \\
$\rr\circ\blud$ % SPASS_yarralumla_splitting_bld_ubl
& 1.9 \\ % 13.7 \\
$\rr\circ\blsp$ %  SPASS_yarralumla_splitting_blu
& 2.4 \\ % 13.7 \\
$\rr\circ\blup$ % SPASS_yarralumla_splitting_blu_ubl
& 1.9 \\ % 13.7 \\
$\sh\circ\rr$ %  SPASS_yarralumla_splitting_sh
& 1.7 \\ % 13.7 \\
$\sh\circ\rr\circ\blsd$ % SPASS_yarralumla_splitting_sh_bld
& 2.3 \\ % 13.7 \\
$\sh\circ\rr\circ\blud$ % SPASS_yarralumla_splitting_sh_bld_ubl
& 1.8 \\ % 13.7 \\
$\sh\circ\rr\circ\blsp$ % SPASS_yarralumla_splitting_sh_blu
& 2.3 \\ % 13.7 \\
$\sh\circ\rr\circ\blup$ % SPASS_yarralumla_splitting_sh_blu_ubl
& 1.7 % 13.7 \\
\end{tabular}
\qquad
\begin{tabular}[t]{lr}
Method & Avg % Max
\\ \hline \hline
$\crr$ % SPASS_yarralumla_splitting_stdrr
& 1.1 \\ % 2.2 \\
$\crr\circ\blsd$ %   SPASS_yarralumla_splitting_stdrr_bld
& 1.7 \\ % 5.6 \\
$\crr\circ\blud$ %   SPASS_yarralumla_splitting_stdrr_bld_ubl
& 1.2 \\ % 3.1 \\
$\crr\circ\blsp$ %   SPASS_yarralumla_splitting_stdrr_blu
& 1.7 \\ % 5.5 \\
$\crr\circ\blup$ %   SPASS_yarralumla_splitting_stdrr_blu_ubl
& 1.2 \\ % 2.7 \\
$\sh\circ\crr$ %   SPASS_yarralumla_splitting_stdrr_sh
& 1.2 \\ % 2.6 \\
$\sh\circ\crr\circ\blsd$ %   SPASS_yarralumla_splitting_stdrr_sh_bld
& 1.9 \\ % 6.0 \\
$\sh\circ\crr\circ\blud$ %   SPASS_yarralumla_splitting_stdrr_sh_bld_ubl
& 1.3 \\ % 3.4 \\
$\sh\circ\crr\circ\blsp$ %   SPASS_yarralumla_splitting_stdrr_sh_blu
& 1.8 \\ % 5.8 \\
$\sh\circ\crr\circ\blup$ %   SPASS_yarralumla_splitting_stdrr_sh_blu_ubl
& 1.3 % 3.0 \\
\end{tabular}
\qquad
\begin{tabular}[t]{lr}
Method & Avg % & Max
\\ \hline \hline
auto %   SPASS38d
& 0 % 0
\end{tabular}
\end{table}

Table~\ref{table_input_size_statistics} gives an impression of the
increase in the size of the input files caused by the transformations.
Although the file sizes were measured after all comments and
white space were removed, variations is name lengths distort the values
slightly (which can be seen in the values for shifting).
The results therefore need to be interpreted cautiously.
The average increase in file size 
does show a significant effect on the size of the problem for the
new range-restriction transformations and also subterm blocking
(both subterm domain blocking and subterm predicate blocking).
The largest increase in size was observed for the problem SYO600-1
(13.7 fold increase), which contained 380 predicate symbols with arity
up to 64, 2 constants and no non-constant function symbols.
The main cause for this increase was the large number of clauses 
added in Step~(4) of the $\rr$~transformation.
For each of the 284 predicate symbols with arity 64 in the problem,
64 clauses were added in Step~(4).
This is a large number.
In contrast for the $\crr$ transformations the increase in size was
negligible, and also, generally, it was significantly lower.
Despite its positive virtues this shows a downside of the $\rr$
transformation.
For problems containing a large number of function symbols with high
arity, Step~(5) similarly adds many clauses, even though the
transformation overall is still effective.

Analysis of the problems solved without any form of blocking revealed
a large number belonged to the Bernays-Sch\"onfinkel class: 
131/176 (74\%) for $\rr$, 132/236 (56\%) for $\sh\circ\rr$, 
133/140 (95\%) for $\crr$, and 134/142 (94\%) for $\sh\circ\rr$.
These results confirmed the expectation that more problems not
solvable with the $\crr$ transformation can be solved with the $\rr$
transformation and the benefits of reducing the number of terms created.

Although the main purpose of BUMG methods is \emph{disproving} theorems and
generating models for satisfiable problems, for completeness we
report in Tables~\ref{table_count_unsat}
and~\ref{table_count_ratings_unsat_summary} the results for
unsatisfiable clausal TPTP problems.
The results were not as uniform as for satisfiable problems.
However some general observations can be made.
SPASS in auto mode fared best overall, and did so in all
TPTP categories and each problem rating category.
For unsatisfiable problems the drawback of BUMG methods is that clauses
need to be exhaustively grounded and each branch in the derivation tree
needs to be closed.
The dominance of SPASS in auto mode is thus not surprising.

For the BUMG methods, 
a general deterioration in performance
could be observed for shifting, when
comparing the results for the groups with baselines $\sh\circ\rr$
and $\sh\circ\crr$ to the respective groups without shifting.
This is plausible because shifting leads to fewer negative literals
in clauses and more positive literals thus reducing the constraining
effect and leading to more splitting.
For problems with higher rating, shifting did seem to have a positive effect;
for instance, in the $(0.40,0.50]$ range, $\sh\circ\rr$ solved
70~problems whereas $\rr$ solved 32 problems.

Within the BUMG groups we expected best performance for the
baseline transformations, because these do not involve blocking and
performing many blocking steps lead to a significant overhead. 
However only for the first group the $\rr$ transformation fared best.
In combination with classical range-restriction $\crr$, somewhat
surprisingly, the best results
were obtained with unrestricted domain blocking, the most expensive form
of blocking, because it is applicable to any terms.
Among the blocking techniques in each case the highest
gain was obtained for unrestricted domain blocking
(see~Table~\ref{table_uniquely_solved_sat_unsat_wrt_baseline}).
However also the greatest loss was observed for this blocking
technique.
The smallest loss and lowest gain was obtained for $\blup$ blocking.
The high loss for $\blud$ could be a reflection of the high increase
in splitting steps preventing quicker detection of contradictions.
Analogously the small loss for $\blup$ could be attributable to the
smallest number of additional splitting steps among the blocking techniques.
The high gain for $\blud$ blocking suggests the inference process 
panned out significantly differently leading to solutions not
found with the other techniques.
This seems to be supported by the results in the third row of
Table~\ref{table_uniquely_solved_sat} according to which, with one
exception, the largest number of uniquely solved problems in each group
was obtained with $\blud$ blocking.
The exception was the first group, where $\rr$ led to the largest number of uniquely
solved problems.
Among all the BUMG methods, $\rr$ solved the largest number of problems
not solved by any of the other methods.
However, these results pale against the number of uniquely solved
problems by SPASS in auto mode.
Only one problem was solved by a BUMG method which was not solved by
SPASS in auto mode.

\begin{sidewaystable}[!htbp]
\caption{Number of problems solved on unsatisfiable clausal TPTP problems.}
\label{table_count_unsat}

\smallskip
\begin{tabular}{l@{\ }r@{\quad\ }r@{\ }r@{\ }r@{\ }r@{\ }r@{\quad\ }r@{\ }r@{\ }r@{\ }r@{\ }r@{\quad\ }r@{\ }r@{\ }r@{\ }r@{\ }r@{\quad\ }r@{\ }r@{\ }r@{\ }r@{\ }r@{\quad\ }r}
Category \rule{0 pt}{40pt}
&    
\# & 
\begin{rotate}{45} $\rr$ \end{rotate} & % SPASS_yarralumla_splitting
\begin{rotate}{45} $\rr\circ\blsd$ \end{rotate} &  % SPASS_yarralumla_splitting_bld
\begin{rotate}{45} $\rr\circ\blud$ \end{rotate} & % SPASS_yarralumla_splitting_bld_ubl
\begin{rotate}{45} $\rr\circ\blsp$ \end{rotate} & %  SPASS_yarralumla_splitting_blu
\begin{rotate}{45} $\rr\circ\blup$ \end{rotate} & % SPASS_yarralumla_splitting_blu_ubl
\begin{rotate}{45} $\sh\circ\rr$ \end{rotate} & %  SPASS_yarralumla_splitting_sh
\begin{rotate}{45} $\sh\circ\rr\circ\blsd$ \end{rotate} & % SPASS_yarralumla_splitting_sh_bld
\begin{rotate}{45} $\sh\circ\rr\circ\blud$ \end{rotate} & % SPASS_yarralumla_splitting_sh_bld_ubl
\begin{rotate}{45} $\sh\circ\rr\circ\blsp$ \end{rotate} & % SPASS_yarralumla_splitting_sh_blu
\begin{rotate}{45} $\sh\circ\rr\circ\blup$ \end{rotate} & % SPASS_yarralumla_splitting_sh_blu_ubl
\begin{rotate}{45} $\crr$ \end{rotate} & % SPASS_yarralumla_splitting_stdrr
\begin{rotate}{45} $\crr\circ\blsd$ \end{rotate} & %   SPASS_yarralumla_splitting_stdrr_bld
\begin{rotate}{45} $\crr\circ\blud$ \end{rotate} & %   SPASS_yarralumla_splitting_stdrr_bld_ubl
\begin{rotate}{45} $\crr\circ\blsp$ \end{rotate} & %   SPASS_yarralumla_splitting_stdrr_blu
\begin{rotate}{45} $\crr\circ\blup$ \end{rotate} & %   SPASS_yarralumla_splitting_stdrr_blu_ubl
\begin{rotate}{45} $\sh\circ\crr$ \end{rotate} & %   SPASS_yarralumla_splitting_stdrr_sh
\begin{rotate}{45} $\sh\circ\crr\circ\blsd$ \end{rotate} & %   SPASS_yarralumla_splitting_stdrr_sh_bld
\begin{rotate}{45} $\sh\circ\crr\circ\blud$ \end{rotate} & %   SPASS_yarralumla_splitting_stdrr_sh_bld_ubl
\begin{rotate}{45} $\sh\circ\crr\circ\blsp$ \end{rotate} & %   SPASS_yarralumla_splitting_stdrr_sh_blu
\begin{rotate}{45} $\sh\circ\crr\circ\blup$ \end{rotate} & %   SPASS_yarralumla_splitting_stdrr_sh_blu_ubl
\begin{rotate}{45} auto \end{rotate} %   SPASS38d
\\
\end{tabular}
\end{sidewaystable}

\begin{sidewaystable}[!htbp]
\caption{Result summary wrt.\ problem rating on unsatisfiable clausal TPTP problems.}
\label{table_count_ratings_unsat_summary}

\smallskip
\begin{tabular}{l@{\quad}r@{\quad\ }r@{\ }r@{\ }r@{\ }r@{\ }r@{\quad\ }r@{\ }r@{\ }r@{\ }r@{\ }r@{\quad\ }r@{\ }r@{\ }r@{\ }r@{\ }r@{\quad\ }r@{\ }r@{\ }r@{\ }r@{\ }r@{\quad\ }r}
Category \rule{0 pt}{40pt}
&    
\# & 
\begin{rotate}{45} $\rr$ \end{rotate} & % SPASS_yarralumla_splitting
\begin{rotate}{45} $\rr\circ\blsd$ \end{rotate} &  % SPASS_yarralumla_splitting_bld
\begin{rotate}{45} $\rr\circ\blud$ \end{rotate} & % SPASS_yarralumla_splitting_bld_ubl
\begin{rotate}{45} $\rr\circ\blsp$ \end{rotate} & %  SPASS_yarralumla_splitting_blu
\begin{rotate}{45} $\rr\circ\blup$ \end{rotate} & % SPASS_yarralumla_splitting_blu_ubl
\begin{rotate}{45} $\sh\circ\rr$ \end{rotate} & %  SPASS_yarralumla_splitting_sh
\begin{rotate}{45} $\sh\circ\rr\circ\blsd$ \end{rotate} & % SPASS_yarralumla_splitting_sh_bld
\begin{rotate}{45} $\sh\circ\rr\circ\blud$ \end{rotate} & % SPASS_yarralumla_splitting_sh_bld_ubl
\begin{rotate}{45} $\sh\circ\rr\circ\blsp$ \end{rotate} & % SPASS_yarralumla_splitting_sh_blu
\begin{rotate}{45} $\sh\circ\rr\circ\blup$ \end{rotate} & % SPASS_yarralumla_splitting_sh_blu_ubl
\begin{rotate}{45} $\crr$ \end{rotate} & % SPASS_yarralumla_splitting_stdrr
\begin{rotate}{45} $\crr\circ\blsd$ \end{rotate} & %   SPASS_yarralumla_splitting_stdrr_bld
\begin{rotate}{45} $\crr\circ\blud$ \end{rotate} & %   SPASS_yarralumla_splitting_stdrr_bld_ubl
\begin{rotate}{45} $\crr\circ\blsp$ \end{rotate} & %   SPASS_yarralumla_splitting_stdrr_blu
\begin{rotate}{45} $\crr\circ\blup$ \end{rotate} & %   SPASS_yarralumla_splitting_stdrr_blu_ubl
\begin{rotate}{45} $\sh\circ\crr$ \end{rotate} & %   SPASS_yarralumla_splitting_stdrr_sh
\begin{rotate}{45} $\sh\circ\crr\circ\blsd$ \end{rotate} & %   SPASS_yarralumla_splitting_stdrr_sh_bld
\begin{rotate}{45} $\sh\circ\crr\circ\blud$ \end{rotate} & %   SPASS_yarralumla_splitting_stdrr_sh_bld_ubl
\begin{rotate}{45} $\sh\circ\crr\circ\blsp$ \end{rotate} & %   SPASS_yarralumla_splitting_stdrr_sh_blu
\begin{rotate}{45} $\sh\circ\crr\circ\blup$ \end{rotate} & %   SPASS_yarralumla_splitting_stdrr_sh_blu_ubl
\begin{rotate}{45} auto \end{rotate} %   SPASS38d
\\
\end{tabular}

Note: ${}^\dagger$1516 = 649 UNS + 28 OPN + 839 UNK
\end{sidewaystable}

\subsection{Findings}

Several findings can be drawn from the results.
The results have confirmed our expectation that unrestricted domain
blocking is a powerful technique, which helps discover
finite models more often than with the other blocking techniques.
The results suggest the technique is indispensable for bottom-up model
generation.
Both in combination with the new range-restricting
transformation, and the classical range-restricting transformation, good
results have been obtained.
Overall, the method based on new range-restriction, shifting and unrestricted
domain blocking performed best on the sample.
On satisfiable problems with higher difficulty rating this method was however
gradually edged out by the method based on classical range-restriction
and unrestricted domain blocking.
This suggests there is a trade-off between the
$\rr$~transformation, which is based on a non-trivial transformation
but does restrict the creation of terms, and the simpler $\crr$
transformation, which has to rely on blocking to restrict
the creation of terms.

The results for subterm domain blocking were good and often not far behind
unrestricted domain blocking for satisfiable problems.
In contrast, predicate blocking seems not to be effective on many
problems.
We attribute this to the nature of the problems in the TPTP library.

An investigation with SPASS-Yarralumla on translations of modal logic
problems has revealed a different picture~\citep{SchmidtStellRydeheard14b}.
There, the best performance was obtained with subterm domain blocking for
both satisfiable and unsatisfiable problems.
Better results than for unrestricted domain blocking were also obtained
with subterm predicate blocking and unrestricted predicate blocking.
Better performances for subterm and predicate blocking are
also expected on problems stemming from (cyclic) description logic
knowledge bases.
Experiments with blocking restricted by excluding a finite subset
of the domain have shown better results than for unrestricted domain
blocking for consistency testing on a
large corpus of ontologies~\citep{KhodadadiSchmidtTishkovsky13a}.
The better performance for restricted forms of blocking on modal and
description logic problems can be attributed to mainstream modal and
description logics having the finite tree model property.
This means every satisfiable formula holds in a model based on a finite
tree, which is not a property of first-order formulae.

The results showed BUMG methods were good for disproving theorems
and generating models for satisfiable problems.
For unsatisfiable problems BUMG methods were however significantly
less efficient than SPASS in auto mode.
For theorem proving purposes a limitation of BUMG methods is that
they require full grounding.
It can be seen already from very small unsatisfiable examples that a
complete BUMG derivation tree can be very large, whereas resolution
proofs are significantly shorter.

Compared to resolution, an advantage of BUMG methods for satisfiable
problems is the division of the search space into branches
which are individually constructed and individually processed.
As a consequence, if the right decisions are made at branching points
models can be found more quickly.
When the branching point decisions are less optimal the performance
can deteriorate dramatically, particularly if the search is trapped in
a branch with only infinite models.
This could be another explanation for the lower success rate of the BUMG
methods observed for more difficult satisfiable problems.
For problems where only infinite models exist, clearly other methods
are better.

\section{Conclusions}
\label{sec:conclusions}

We have presented and tested a number of enhancements for BUMG methods.
An important aspect is that our enhancements exploit the strengths
of readily available BUMG system with only modest modifications.
Our range restriction technique is a refinement of existing
transformations to range-restricted clauses in that terms are added
to the domain of interpretation on a `by need' basis.
Moreover, we have presented methods that allow us to extend BUMG
methods with blocking techniques related to loop checking
techniques with a long history in the more specialized setting of
modal and description logics.

The experimental evaluation has shown blocking techniques are indispensable in BUMG
methods for satisfiable problems.
In particular, unrestricted domain blocking turned out to be the most
powerful technique on problems from the TPTP library.
Limiting the creation of terms during the inference process by using the
new range restricting transformation paid off, leading to better results.
It is particularly advisable together with the shifting transformation.
The experimental results however also show that classical range
restriction together with unrestricted blocking is a good complementary
method.
Because model generation methods are not just aimed at showing the
existence of models but are built to construct and return models, when
no models exists the entire search space must be traversed, which 
has led to inferior performance compared to saturation-based resolution.

Our bottom-up model generation approach is especially suitable for generating small models and
it is possible to show the approach using unrestricted domain blocking
allows us to compute finite models when they exist.
The models produced by subterm blocking and predicate blocking are not
as small as those produced by unrestricted domain blocking.  
In particular, the generated models do not need to be Herbrand models.
It follows from how the transformations work that the generated models
are quasi-Herbrand models, in the following sense. Whenever $\dom(s)$ and
$\dom(t)$ hold in the (Herbrand) model constructed by the
BUMG method, then (as in Herbrand
interpretations) the terms $s$ and $t$ are mapped to themselves in the
associated (possibly non-Herbrand) model. Reconsidering the example in
the Introduction of the two unit clauses $\ID P(\ID a)$ and $\ID
Q(\ID b)$, the associated model maps $\ID a$ and $\ID b$ to
themselves, regardless as to which transformations are applied as
long as it includes a form of subterm blocking.
In this way, more informative models are produced than those computed
by, for example, MACE- and SEM-style finite model searchers (and also
unrestricted domain blocking).
From an applications perspective, this can be an advantage because larger
models are more likely to be helpful to a user debugging mistakes
in the formal specification of a program or protocol, or an ontology
engineer trying to discover why an expected entailment does not follow
from an ontology.

Research in automated theorem proving on developing decision procedures
has concentrated on developing refinements of resolution, mainly
ordering refinements, for deciding solvable fragments of first-order
logic.
Fragments decidable with ordered resolution are complementary to the
fragments that can be decided by refinements using the techniques
presented in this paper.
We have thus extended the set of techniques available for resolution methods
to turn them into more effective and efficient (terminating) automated
reasoning methods.
In particular, we have shown that all procedures based on
hyperresolution, or BUMG methods, can decide the Bernays-Sch\"onfinkel class
and the class of \BS clauses with equality.

Studying how well the ideas and techniques discussed in this paper
can be exploited and behave in dedicated BUMG provers, tableau-based
provers and other provers (including resolution-based provers) is
very important but is beyond the scope of the present paper.
Initial results with another prover,
Darwin~\citep{Baumgartner:etal:Darwin:IJAIT:2006},
are very encouraging.
An in-depth comparison and analysis of BUMG approaches with our
techniques and MACE-style or SEM-style model generation
would also be of interest.
Another source for future work is to combine the presented transformations
with other BUMG techniques, such as 
magic sets transformations~\citep{HasegawaInoueOhtaKoshimura97,Stickel94}, 
a typed version of range-restriction~\citep{Baumgartner:Furbach:Stolzenburg:RestartMEAnswers:AI:97},
and minimal model
computation~\citep{BryYahya00,BryTorge98,PapacchiniSchmidt11}.
Having been designed to be generic, we believe that our transformations
carry over to formalisms with default negation, which could provide
a possible basis for enhancements to answer-set programming systems.

\paragraph{Acknowledgements.}
The second author is grateful to Christoph Weidenbach and Uwe
Waldmann for hosting her during 2010 and 2013--2014.
In this time the implementation of SPASS-yarralumla was completed
and the experimental evaluation was undertaken on the cluster of the
Max-Planck-Institut f\"ur Informatik, Saarbr\"ucken.
We thank Uli Furbach, Dmitry Tishkovsky, Uwe Waldmann and Christoph Weidenbach for
useful discussions and comments on this research.
This work was supported by the UK Engineering and Physical Sciences
Research Council (EPSRC) (grants EP/F068530/1 and EP/H043748/1),
NICTA, Canberra, Australia, and the Max-Planck-Institut, Saarbr\"ucken, Germany.

\bibliographystyle{abbrvnat}
%\bibliography{bib}

\end{document}